\documentclass{article}

\usepackage[preprint]{neurips_2025}

\usepackage[utf8]{inputenc}
\usepackage[T1]{fontenc}
\usepackage[colorlinks,citecolor=orange]{hyperref}
\usepackage{url}
\usepackage{booktabs}
\usepackage{amsfonts}
\usepackage{nicefrac}
\usepackage{microtype}
\usepackage{xcolor}
\usepackage{pifont}

\usepackage{graphicx}
\usepackage{wrapfig}
\usepackage{caption}
\usepackage{subcaption}
\usepackage{amsmath}
\usepackage{amssymb}
\usepackage{mathtools}
\usepackage{amsthm}
\usepackage{multirow}
\usepackage{bbm}

\usepackage{algorithm}
\usepackage{algorithmic}
\usepackage{setspace}
\usepackage[capitalize,noabbrev]{cleveref}

\usepackage{pifont}
\newcommand{\DAGR}{\textsc{Dagr}}
\newcommand{\DGCA}{\textsc{Dgca}}

\newcommand{\RR}{\mathbb{R}}
\newcommand{\EE}{\mathbb{E}}

\newcommand{\cM}{\mathcal{M}}
\newcommand{\cS}{\mathcal{S}}
\newcommand{\cA}{\mathcal{A}}
\newcommand{\cG}{\mathcal{G}}
\newcommand{\cZ}{\mathcal{Z}}
\newcommand{\cD}{\mathcal{D}}
\newcommand{\cL}{\mathcal{L}}

\newcommand{\cH}{\mathcal{H}}
\newcommand{\cF}{\mathcal{F}}

\newcommand{\cE}{\mathcal{E}}
\newcommand{\indep}{\perp\perp}

\usepackage{aliascnt}

\theoremstyle{plain}
\newtheorem{theorem}{Theorem}[section]

\newaliascnt{proposition}{theorem}
\newtheorem{proposition}[proposition]{Proposition}
\aliascntresetthe{proposition}

\newaliascnt{assumption}{theorem}
\newtheorem{assumption}[assumption]{Assumption}
\aliascntresetthe{assumption}

\newaliascnt{lemma}{theorem}
\newtheorem{lemma}[lemma]{Lemma}
\aliascntresetthe{lemma}

\newaliascnt{corollary}{theorem}
\newtheorem{corollary}[corollary]{Corollary}
\aliascntresetthe{corollary}

\theoremstyle{remark}
\newaliascnt{remark}{theorem}
\newtheorem{remark}[remark]{Remark}
\aliascntresetthe{remark}

\theoremstyle{definition}
\newtheorem{definition}{Definition}

\crefname{theorem}{Theorem}{Theorems}
\Crefname{theorem}{Theorem}{Theorems}
\crefname{proposition}{Proposition}{Propositions}
\Crefname{proposition}{Proposition}{Propositions}
\crefname{assumption}{Assumption}{Assumptions}
\Crefname{assumption}{Assumption}{Assumptions}
\crefname{lemma}{Lemma}{Lemmas}
\Crefname{lemma}{Lemma}{Lemmas}
\crefname{remark}{Remark}{Remarks}
\Crefname{remark}{Remark}{Remarks}
\crefname{corollary}{Corollary}{Corollaries}
\Crefname{corollary}{Corollary}{Corollaries}
\crefname{definition}{Definition}{Definitions}
\Crefname{definition}{Definition}{Definitions}

\title{\DAGR{}: State-Conditioned Goal Representations via Difference-Aware Goal Cross-Attention}

\author{%
  Xing Lei\textsuperscript{1}\quad
  Wenyan Yang\textsuperscript{2}\quad
  Xuetao Zhang\textsuperscript{1}\thanks{Corresponding author.}\quad
  Donglin Wang\textsuperscript{3} \\[3pt]
  \textsuperscript{1}\,
  Institute of Artificial Intelligence and Robotics, Xi'an Jiaotong University\\
  \textsuperscript{2}\,Department of Electrical Engineering and Automation, Aalto University\\
  \textsuperscript{3}\,School of Engineering, Westlake University\\[3pt]
  \texttt{leixing@stu.xjtu.edu.cn}
}

\begin{document}

\maketitle

\begin{abstract}
Goal-conditioned reinforcement learning hinges on how the goal is encoded. Contrastive, metric, temporal-distance and information-theoretic encoders disagree on the objective. They agree on one thing. None of them sees the current state, so the embedding cannot mark which part of the goal still needs action, and the policy must recover that cue by inverting both encoders. We propose \DAGR{}, which refines the static embedding of any late-fusion encoder into a state-conditioned one through multi-scale gated cross-attention. A gated residual holds the refinement near the base, and a difference-aware attention rule biases the scores by a per-token state-goal mismatch. A single condition decides what such a refinement can guarantee, namely whether the block returns its input at closed gates. We prove that the usual post-norm placement violates it, measure the consequence on frozen checkpoints, and recover part of the resulting loss by restoring the condition. On OGBench \DAGR{} improves navigation and matches or trails the base elsewhere. Our ablations trace the gain to the gated residual rather than to the difference bias that names the method. Code is available at \url{https://github.com/leixingxing1/DAGR}.
\end{abstract}

\section{Introduction}
\label{sec:intro}
Goal-conditioned reinforcement learning trains agents that reach any specified goal state~\citep{schaul2015universal,plappert2018multi,liu2022goal}. In the offline setting~\citep{lange2012batch,levine2020offline}, the learned goal representation governs whether the policy generalizes to unseen state-goal pairs~\citep{yang2022rethinking,ma2022offline,park2025ogbench}. Recent methods disagree sharply on how to shape it. Contrastive learning~\citep{eysenbach2022contrastive,zheng2024contrastive}, metric learning~\citep{ma2023vip,wang2023optimal,park2024hilp}, temporal-distance modeling~\citep{park2026dual} and information-theoretic compression~\citep{alemi2017deep,shah2021recon} pull the embedding in different directions. They agree on one structural point. The state encoder $\psi$ and the goal encoder $\phi$ train apart, and their outputs meet only inside the policy and value heads~\citep{park2025ogbench}. We call the resulting $\phi(g)$ a \emph{state-independent} goal representation. That independence keeps the objectives tractable~\citep{park2026dual}, but it carries a cost. Because $\phi(g)$ never sees the current state, it cannot indicate which aspects of the goal still call for action, and the policy must recover that cue by inverting both encoders. The bottom panel of \cref{fig:concept} makes this concrete. One goal hands every state the same embedding, whatever action that state still requires.

\begin{figure*}[t]
\vspace{-8pt}
\centering
\includegraphics[width=0.89\textwidth]{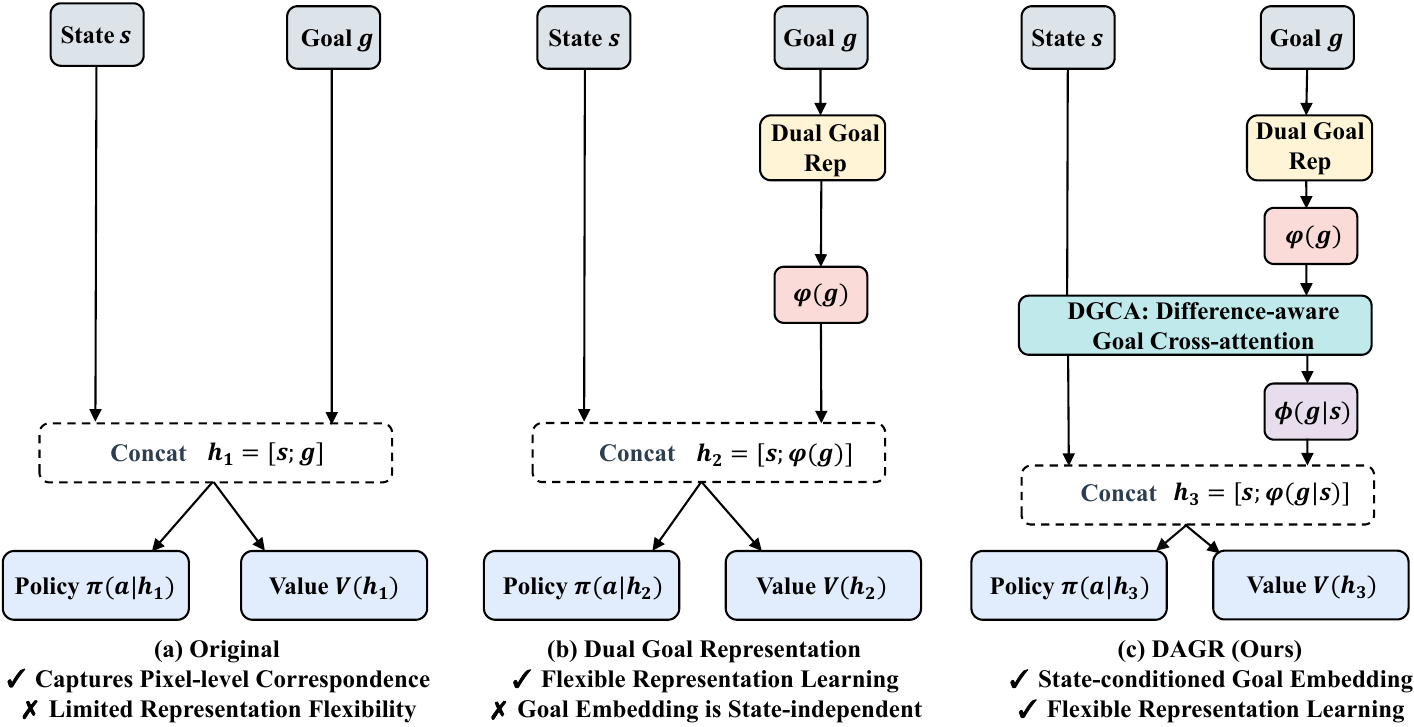}\\[4pt]
\includegraphics[width=0.66\textwidth]{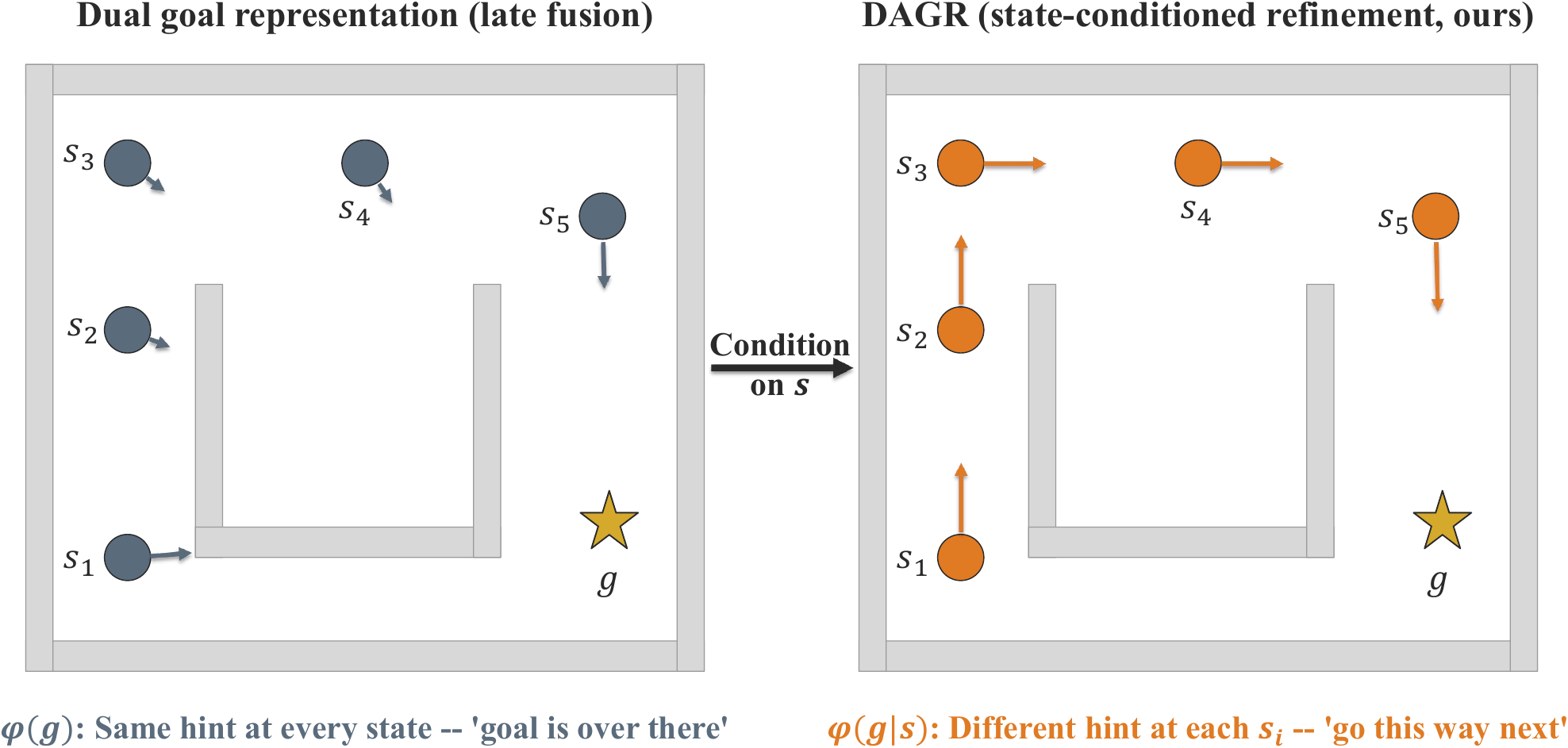}
\caption{\textbf{From a state-independent goal embedding to a state-conditioned one.} The top panel contrasts the data flows of late fusion, Dual~\citep{park2026dual} and \DAGR{}. The bottom panel shows the consequence in one maze with a shared goal $g$ and states $s_1$ to $s_5$. Dual gives every $s_i$ the same hint. \DAGR{} makes it reflect the discrepancy between $s_i$ and $g$. \cref{eq:dagr_discrete} states this formally.}
\vspace{-12pt}
\label{fig:concept}
\end{figure*}
We propose \DAGR{}, a module that refines the static $\phi(g)$ of any late-fusion encoder into a state-conditioned $\phi(g \mid s)$ through cross-attention. Two choices keep the refinement safe. The first is a multi-scale gated residual whose branches start inert, so training begins from the block's skip path and the downstream losses open a gate only where doing so pays. The second is an attention rule that adds a learnable non-negative bias, read off a per-token difference map, to the usual similarity scores. Refining a representation that already carries guarantees raises one question above all others. Does the refined block still return the base when its gates close? \cref{sec:method} shows that the answer is a design rule rather than a detail. It holds for a pre-norm gated residual and fails for the post-norm one we run, and every guarantee we can state depends on which of the two is used.
We evaluate \DAGR{} on OGBench~\citep{park2025ogbench}, layered on Dual~\citep{park2026dual} with GCIVL~\citep{park2025ogbench} downstream. On state-based navigation it improves the success rate on every task but one, and the gains are largest where the base is weak but above zero. The exception has a state space too small to compress, which is the regime our structural condition marks as uninformative. On visual navigation it leads every goal-representation baseline, so the directional signal survives the pixel encoder. On manipulation and on the discrete puzzles it matches the base on some tasks and falls below it on others, and the state-based manipulation tasks where it regresses are the ones that violate that condition. Our ablations then ask which component produces the gains, and the answer is not the one the method is named after. The gated residual carries the improvement. Removing the difference bias costs navigation under four points.

\textbf{Contributions.}
\begin{itemize}
\item We isolate the state-independence of the goal encoder as a design axis orthogonal to the representation objective, and bound its cost in information and in Rademacher complexity. Both bounds serve as motivation.
\item We propose \DAGR{}, a module that state-conditions any late-fusion goal encoder. We identify the single condition on which its guarantees rest, the closed-gate identity, and prove that post-norm violates it and pre-norm satisfies it.
\item We measure the consequence on frozen checkpoints. Pre-norm reproduces it to seven digits, post-norm retains $1.8\%$ of the base norm, and the contraction hypothesis behind sufficiency preservation fails under both.
\item We state a structural condition on the task at a fixed compression budget. It predicts where state-conditioning helps, which holds on OGBench, and restoring the identity recovers eight of twenty-five points lost elsewhere.
\item We attribute the gain component by component and report that the difference bias, which names the module, acts at initialization and little beyond.
\end{itemize}
\section{Related Work}
\label{sec:related}
\textbf{Offline goal-conditioned RL.} Offline GCRL learns goal-reaching policies from fixed data~\citep{lange2012batch,levine2020offline} through behavioral cloning, value-based estimation, hierarchical planning, sequence modeling, contrastive value learning and generative trajectory modeling~\citep{lynch2020learning,ghosh2021learning,kostrikov2022offline,park2025ogbench,park2023hiql,ahn2025option,zhou2025flattening,lei2025gchr,giammarino2026goal,lei2026qhyer,eysenbach2022contrastive,myers2024learning,zheng2024contrastive,eysenbach2019search,wang2024goplan,ma2022far,sikchi2024dual,jain2024learning,bao2026nftr}. OGBench~\citep{park2025ogbench} shows that no single method dominates across task categories. \DAGR{} is orthogonal to this axis, since it refines the goal representation that any such algorithm consumes.

\begin{wrapfigure}[24]{r}{0.4\linewidth}    
\vspace{-0.5em}    
\centering	
\centerline{\includegraphics[width=0.45\textwidth]{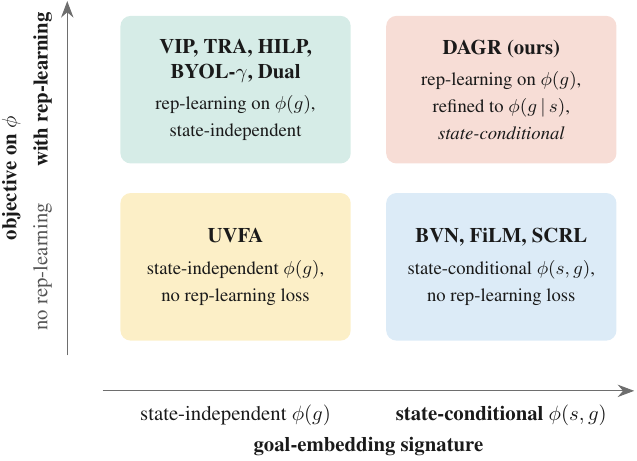}}    
\caption{\textbf{Two design axes for goal encoders in offline GCRL.}Horizontal: whether $\phi$ is state-conditional ($\phi(s,g)$ or $\phi(g\,|\,s)$)or state-independent ($\phi(g)$). Vertical: whether $\phi$ is trained with arepresentation-learning objective. Among these, \DAGR{} is the only one that satisfies both.Quasimetric methods act on a third orthogonal axis (constraining the value-function class) and are omitted here.}    
\label{fig:design_axes}
\end{wrapfigure}
\textbf{Goal representation learning.} Metric methods align geometric and temporal distances~\citep{ma2023vip,park2024hilp,wang2023optimal,myers2025offline,zheng2026scaling}. Temporal-distance methods fit the embedding to a value or distance signal~\citep{myers2025temporal,park2026dual}. Information-theoretic methods compress it through a variational bottleneck~\citep{tishby2000information,alemi2017deep,shah2021recon,park2023hiql}, and self-predictive methods use bootstrapped temporal consistency~\citep{lawson2025self,schwarzer2020data}. Each of them yields a state-independent $\phi(g)$. \DAGR{} acts on that output rather than on the objective behind it, so the two are composable.

\textbf{Two axes prior work conflates.} The objective axis above keeps $\phi$ a function of $g$ alone. A second axis asks whether $\phi$ is state-conditional. UVFA~\citep{schaul2015universal} keeps it independent, BVN~\citep{yang2022bilinear} uses $\phi(s, g)$ and is strictly more expressive, and FiLM~\citep{perez2018film} together with SCRL~\citep{zheng2024stabilizing} condition on $s$ without any representation objective. \DAGR{} sits closest to BVN and differs in three respects. It refines a representation-learned base instead of training $\phi(s, g)$ monolithically. It is built to reduce to that base at initialization, under the condition \cref{prop:placement} isolates. It biases attention with an explicit difference map. Quasimetric value classes~\citep{Pitis2020An,wang2022on,wang2022improved,wang2023optimal} form a third axis acting on the value head rather than on $\phi$. \cref{fig:design_axes} places the first two side by side.

\textbf{Cross-Attention in Computer Vision and RL.}Cross-attention has become a standard mechanism for selective information flow across modalities in computer vision~\citep{carion2020end,jaegle2021perceiver,alayrac2022flamingo}, with recent variants modifying the attention rule itself such as Gated Attention~\citep{qiu2026gated}. In RL, it underlies trajectory modeling~\citep{chen2021decision,reed2022generalist}, in-context value inference~\citep{xu2026incontext}, and an expanding body of work on vision-language-action models~\citep{brohan2023rt,brohan2023rt2,team2024octo,kim2024openvla,black2024pi0,wang2024hpt,huang2025early,zhong2026acot}. FiLM~\citep{perez2018film,hill2020environmental} applies feature-wise modulation in goal-conditioned RL but is spatially uniform. To our knowledge, \DAGR{} is the first to introduce a learnable difference bias into goal-conditioned cross-attention and to deploy it at multiple scales over a temporal-distance goal representation.

\section{Preliminaries}
\label{sec:prelim}

\subsection{Offline Goal-Conditioned RL}

A goal-conditioned MDP (GCMDP)~\citep{kaelbling1993learning} is a tuple $\cM = (\cS, \cA, \cG, p, r, \gamma)$ with goal space $\cG \subseteq \cS$, sparse reward $r(s, g) = \mathbbm{1}[s = g]$, and discount $\gamma \in (0, 1)$. A policy $\pi : \cS \times \cG \to \Delta(\cA)$ induces value $V^\pi(s, g) = \EE_\pi[\sum_t \gamma^t r(s_t, g) \mid s_0 = s]$ with optimum $V^*(s, g)$. We define the optimal temporal distance $d^*(s, g) = \log_\gamma V^*(s, g)$, which equals the shortest-path length in deterministic environments. The offline dataset $\cD = \{(s_i, a_i, s'_i)\}_{i=1}^N$ is fixed.

\subsection{Late Fusion and the State-Independence of $\phi(g)$}
\label{sec:prelim_latefusion}

We focus on the late-fusion paradigm shared by current goal-representation methods (cf.\ \cref{fig:concept}, left). A state encoder $\psi : \cS \to \RR^{d_s}$ and a goal encoder $\phi : \cG \to \RR^d$ are trained independently, and the policy is parameterized as $\pi_\theta(a \mid \psi(s), \phi(g))$. The defining property of late fusion is that the goal encoder $\phi$ takes no state input. For any fixed goal $g$, the same vector $\phi(g)$ is supplied to the policy and value heads regardless of the current state $s$. The policy network must therefore infer at each step which
components of $\phi(g)$ are currently actionable. \DAGR{} replaces $\phi(g)$ by a state-conditioned $\phi(g \mid s)$ that exposes this information explicitly at the representation level.\footnote{Throughout the paper, $\phi(g\mid s)$ denotes a \emph{deterministic} function $\phi:\cG\times\cS\to\RR^d$, with the vertical bar reading ``$\phi$ of $g$, parameterized by $s$.'' This is a notational convention and should not be read as a conditional distribution.}

\begin{definition}[Sufficient Goal Representation~\citep{park2026dual}]
\label{def:sufficiency}
A representation $\phi : \cG \to \RR^d$ is sufficient for optimal control if there exists $\pi^* : \cS \times \RR^d \to \Delta(\cA)$ such that $V^{\pi^*(\cdot \mid s, \phi(g))}(s, g) = V^*(s, g)$ for all $(s, g) \in \cS \times \cG$.
\end{definition}
\subsection{Multi-Head Cross-Attention}\label{sec:prelim_ca} Given $Q \in \RR^{n_q \times d}$, $K \in \RR^{n_k \times d}$, $V \in \RR^{n_k \times d_v}$, cross-attention is $\mathrm{CrossAttn}(Q, K, V) = \mathrm{softmax}(QK^\top / \sqrt{d})\,V$~\citep{vaswani2017attention}. Multi-head attention computes $H$ such operations in parallel with distinct projections $W_h^Q, W_h^K, W_h^V$ and concatenates them through $W^O$. The softmax weights measure similarity, so attention concentrates on tokens that resemble the query. For goal-conditioned control, the policy must act on the components in which state and goal disagree rather than on the components in which they agree. This mismatch between the inductive bias of standard cross-attention and the requirements of goal-conditioned control motivates the difference-bias term introduced in \cref{sec:method}.

\section{Difference-Aware Goal Representations}
\label{sec:method}
We begin with the core idea of \DAGR{}, following the exposition pattern of the Dual representation~\citep{park2026dual}. Recall from \cref{sec:prelim} that the optimal temporal distance $d^*(s, g) = \log_\gamma V^*(s, g)$ corresponds in deterministic environments to the shortest-path length from $s$ to $g$. The Dual representation characterizes a goal $g$ by the set of optimal temporal distances $d^*(\cdot, g)$ from all other states. In a discrete state space $\cS = \{s_1, \ldots, s_K\}$, this corresponds to the vector
\begin{equation}
\varphi^{\vee}(g) \;=\; \big[d^*(s_1, g), \, d^*(s_2, g), \, \ldots, \, d^*(s_K, g)\big]^{\top},
\label{eq:dual_discrete}
\end{equation}
and in general to the functional $\varphi^{\vee}(g)(s') = d^*(s', g)$, which depends on $g$ alone. \DAGR{} extends this construction with a state-dependent weighting:
\begin{equation}
\varphi_{\textup{\DAGR{}}}^{\vee}(g \mid s) \;=\; \big[d^*(s_1, g)\cdot\Delta_{s, g}(s_1), \, \ldots, \, d^*(s_K, g)\cdot\Delta_{s, g}(s_K)\big]^{\top},
\label{eq:dagr_discrete}
\end{equation}
in the discrete case, and to the functional $\varphi_{\textup{\DAGR{}}}^{\vee}(g \mid s)(s') = d^*(s', g) \cdot \Delta_{s, g}(s')$ in general. Here $\Delta_{s, g} : \cS \to [0, 1]$ is large when $s'$ is informative about the difference between $s$ and $g$ and small when $s'$ is irrelevant. We call $\varphi_{\textup{\DAGR{}}}^{\vee}$ the \emph{difference-aware goal representation} of the GCMDP $\cM$. Two symbols run in parallel from here on. The functional $\varphi_{\textup{\DAGR{}}}^{\vee}(g \mid s)$ is the idealized object just defined, and the vector $\phi_{\textup{\DAGR{}}}(g \mid s) \in \RR^d$ is the network that realizes it. \cref{sec:dgca} builds that network and approximates $\Delta_{s, g}$ by a per-token difference map.

\begin{figure*}[t]
\vspace{-3pt}
\centering
\includegraphics[width=0.82\textwidth]{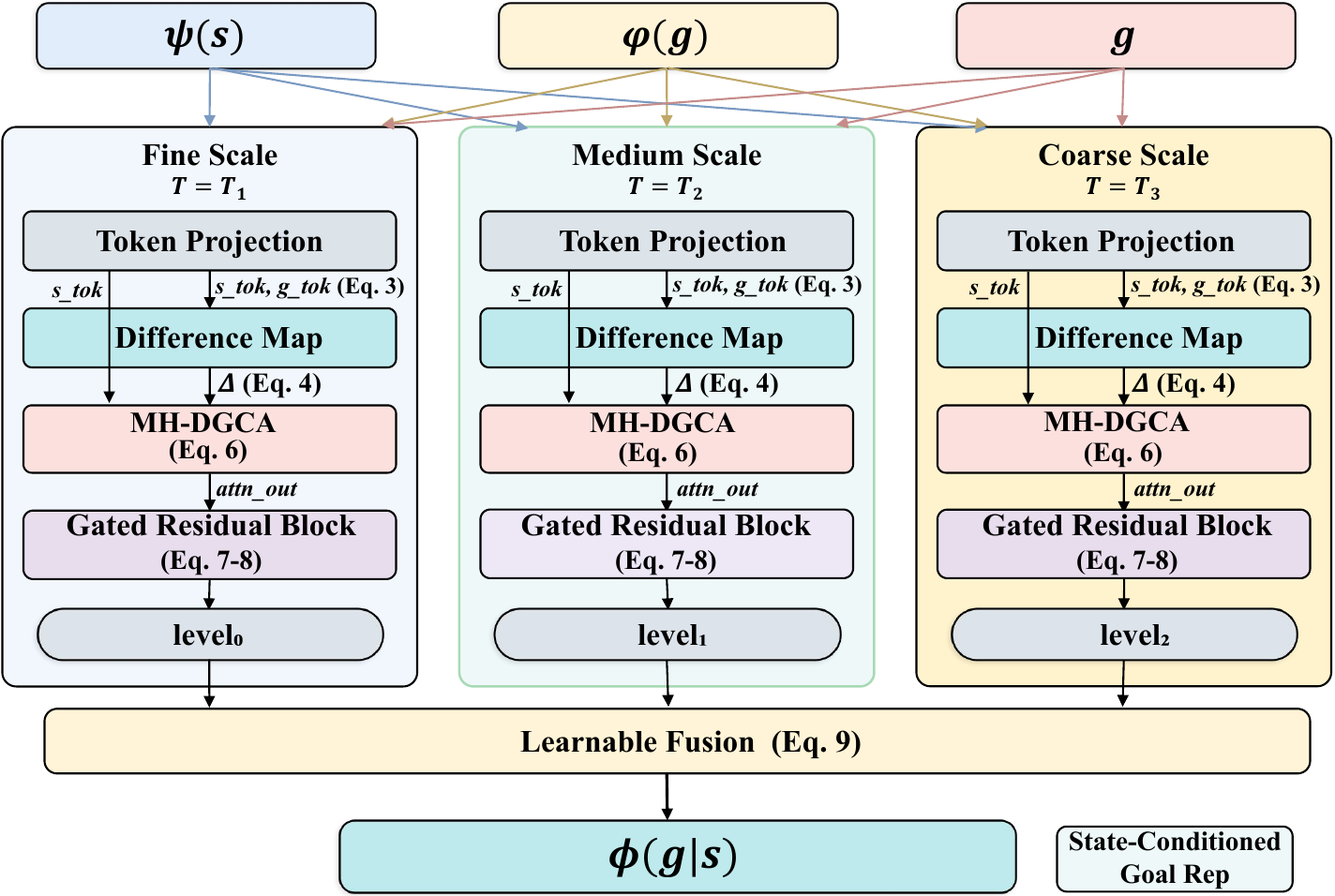}
\caption{\textbf{Multi-scale \DGCA{} architecture.} Three scale levels with token counts $T_\ell \in \{16, 8, 4\}$ project the flat encoder outputs $\psi(s)$ and $\psi_g(g)$ into aligned pseudo-token spaces. Each level computes a difference map $\Delta^{(\ell)}$, runs multi-head \DGCA{} with $\phi(g)$ as the query, and produces a gated residual update. Outputs are combined through learnable fusion weights to yield $\phi_{\textup{\DAGR{}}}(g \mid s)$.}
\vspace{-0.5cm}
\label{fig:architecture}
\end{figure*}
The rest of this section unpacks this construction. We first show what a state-independent goal representation cannot encode (\cref{sec:problem}), then specify the single-scale (\cref{sec:dgca}) and multi-scale (\cref{sec:multiscale}) DGCA blocks that realize $\Delta_{s, g}$, and conclude with the theoretical properties of the resulting representation (\cref{sec:theory}). The joint training procedure with the downstream offline GCRL algorithm is deferred to \cref{app:training}.

\subsection{The Representation-Level Bottleneck of Late Fusion}
\label{sec:problem}

To make precise what is missing in a state-independent $\phi(g)$, let $S$, $G$, and $A^*$ denote the random variables for state, goal, and optimal action. Under late fusion, conditioned on $S = s$ the optimal action $A^*$ is determined by $\pi^*(\cdot \mid s, g)$, which depends on $g$ only through $\phi(g)$. This yields the Markov chain $G \to \phi(G) \to A^*$ given $S = s$, from which the data processing inequality yields the following bound.

\begin{proposition}[Information Bottleneck of State-Independent Goal Representations]
\label{rem:info_bottleneck}
For any state-independent $\phi$,
\[
I(A^* ; \phi(G) \mid S) \;\leq\; I(A^* ; G \mid S),
\]
with equality if and only if $\phi(G)$ retains all goal information relevant to $A^*$ given $S$.
\end{proposition}

This inequality alone does not separate the two designs, since for each fixed $s$ a state-conditioned $\phi(\cdot \mid s)$ is also a deterministic function of $G$ and obeys the same bound. The difference lies in the equality condition. A state-independent encoder must meet it with one map, simultaneously for every $s$; a state-conditioned one may use a different map per $s$. Where no single $d$-dimensional map suffices for all $s$ at once, the inequality is strict for every state-independent $\phi$. A canonical example is when $\pi^*(a \mid s, g) = f(s - g)$, the relational case. Even when $\phi$ is information-theoretically sufficient, the downstream policy still pays a sample-complexity price. Realizing $\pi^*$ from inputs $(\psi(s), \phi(g))$ requires the network to implicitly invert both encoders before computing the difference. This composition is strictly harder than receiving $s - g$ directly.

\begin{proposition}[Late Fusion Admits a Looser Complexity Bound on Relational Tasks]
\label{rem:late_fusion_failure}
Suppose $\pi^*(a \mid s, g) = f(s - g)$. Then realizing $\pi^*$ from the late-fusion inputs $(\psi(s), \phi(g))$ requires composing $f$ with measurable inverse selections of $\psi$ and of $\phi$. The composition property of Rademacher complexity~\citep{bartlett2002rademacher} then yields an upper bound on the policy class that exceeds the corresponding bound under direct access to $s - g$ by a factor $L_{\psi^{-1}} L_{\phi^{-1}} \geq 1$.
\end{proposition}
We are precise about what this does and does not say. It compares two \emph{upper bounds}. It does not establish that the true Rademacher complexity of the realized function class is larger under late fusion. The proposition motivates the architecture. It does not prove that late fusion must be worse.
Both observations point to the same fix: surface state-conditioned goal information at the representation level rather than asking the policy to recover it. This motivates the module we introduce next. Full proofs are in \cref{app:proof_info_bottleneck,app:proof_late_fusion_failure}.

\subsection{Single-Scale Difference-Aware Goal Cross-Attention (\DGCA{})}
\label{sec:dgca}
The module takes the static $\phi(g) \in \RR^d$, the state encoding $\psi(s) \in \RR^{d_s}$ and a goal-side encoding $\psi_g(g) \in \RR^{d_s}$, and returns $\phi(g \mid s) \in \RR^d$. On visual tasks $\psi_g$ is the goal branch of the pixel encoder, and on state-based tasks it is the identity. Two properties guide the design. The module should reduce to the identity on $\phi(g)$ at initialization, so training starts from the base. It should also be able to route state features that disagree with the goal into the refined vector, trained end to end by the downstream losses. The block is a multi-head cross-attention in which the goal supplies the query and the state supplies keys and values, with an additive bias from a per-token difference map. \cref{fig:architecture} shows it. Three ingredients realize it.

The first is a token decomposition. Let $T$ denote the pseudo-token count, $H$ the number of heads, $d_k$ the per-head key dimension, and $d_m := H d_k$. With separately learned projections $W_s, W_g \in \RR^{(T d_m) \times d_s}$,
\begin{equation}
\mathbf{s}_\mathrm{tok} = \mathrm{Reshape}\big(W_s\, \psi(s)\big),
\qquad
\mathbf{g}_\mathrm{tok} = \mathrm{Reshape}\big(W_g\, \psi_g(g)\big)
\;\in\; \RR^{T \times d_m}.
\label{eq:tokens}
\end{equation}
The two projections are untied, which aligns the token spaces position-wise and renders the per-token difference $\mathbf{s}_\mathrm{tok}^{(t)} - \mathbf{g}_\mathrm{tok}^{(t)}$ meaningful.

The second ingredient is the difference map. A normalized $\ell_2$ difference measures the mismatch at each token position,
\begin{equation}
\Delta_t \;=\; \frac{\|\mathbf{s}_\mathrm{tok}^{(t)} - \mathbf{g}_\mathrm{tok}^{(t)}\|_2}{\max_{t'} \|\mathbf{s}_\mathrm{tok}^{(t')} - \mathbf{g}_\mathrm{tok}^{(t')}\|_2 + \varepsilon} \;\in\; [0, 1],
\qquad
\Delta = (\Delta_1, \ldots, \Delta_T),
\label{eq:diff_map}
\end{equation}
with $\varepsilon = 10^{-6}$ for numerical stability, matching the released implementation. The maximum is taken within the same $(s, g)$ pair, so $\Delta$ ranks tokens by relative mismatch and is invariant to the absolute scale of state-goal differences. A value near one marks a position at which state and goal still differ, and a value near zero marks a position at which they agree.

The third ingredient is the attention rule that uses $\Delta$ as an additive bias in logit space.

\begin{definition}[Difference-Aware Goal Cross-Attention]
\label{def:dgca}
Given a goal query $Q \in \RR^{1 \times d_k}$, state keys and values $K, V \in \RR^{T \times d_k}$, and a difference map $\Delta \in [0, 1]^T$ as defined in \cref{eq:diff_map},
\begin{equation}
\mathrm{DGCA}(Q, K, V, \Delta) \;=\; \mathrm{softmax}\left(\frac{Q K^\top}{\sqrt{d_k}} + \zeta(\lambda) \cdot \Delta\right) V \;\in\; \RR^{1 \times d_k},
\label{eq:dgca}
\end{equation}
where $\zeta : \RR \to \RR_{\geq 0}$ is defined by $\zeta(x) := \mathrm{softplus}(x) = \log(1 + e^x)$, and $\lambda \in \RR$ is a learnable scalar parameter. We initialize $\lambda_0 = -5$, which yields $\zeta(\lambda_0) = \log(1 + e^{-5}) \approx 0.0067$.
\end{definition}

The reparameterization through $\zeta$ keeps the bias non-negative, so higher token-wise discrepancy only increases the corresponding attention logit. At initialization, $\zeta(\lambda_0)$ is negligible and \DGCA{} reduces to standard cross-attention. As $\lambda$ grows during training, the bias additively boosts the logits of tokens with high $\Delta_t$. A per-head $\lambda_h$ lets different heads specialize at different bias intensities.

The multi-head version applies $H$ parallel heads with projections $W_h^Q \in \RR^{d \times d_k}$ for the goal query and $W_h^K, W_h^V \in \RR^{d_m \times d_k}$ for the state tokens,
\begin{equation}
\mathrm{head}_h = \mathrm{softmax}\left(\tfrac{\phi(g)W_h^Q (\mathbf{s}_\mathrm{tok}W_h^K)^\top}{\sqrt{d_k}} + \zeta(\lambda_h)\Delta\right)\mathbf{s}_\mathrm{tok}W_h^V,
\quad
\mathrm{attn\_out} = \mathrm{Concat}_h(\mathrm{head}_h)\, W^O,
\label{eq:mhdgca}
\end{equation}
with $W^O \in \RR^{(H d_k) \times d}$ and a per-head $\lambda_h$ that permits different heads to adopt different bias intensities.
The attention output enters through a gated residual block with a feed-forward network. Where the normalization sits is a free choice, and it will turn out to matter, so we name both. The \emph{post-norm} placement, which we run, normalizes each sub-block output,
\begin{align}
\mathbf{x} &= \mathrm{LayerNorm}\big(\phi(g) + \sigma(\alpha_\mathrm{attn}) \odot \mathrm{attn\_out}(\phi(g))\big),
\label{eq:gated_res_attn} \\
\mathrm{out} &= \mathrm{LayerNorm}\big(\mathbf{x} + \sigma(\alpha_\mathrm{ffn}) \odot \mathrm{FFN}(\mathbf{x})\big),
\label{eq:gated_res_ffn}
\end{align}
whereas the \emph{pre-norm} placement moves both normalizations inside the gated branches,
\begin{align}
\mathbf{x} &= \phi(g) + \sigma(\alpha_\mathrm{attn}) \odot \mathrm{attn\_out}\big(\mathrm{LayerNorm}(\phi(g))\big),
\label{eq:prenorm_attn} \\
\mathrm{out} &= \mathbf{x} + \sigma(\alpha_\mathrm{ffn}) \odot \mathrm{FFN}\big(\mathrm{LayerNorm}(\mathbf{x})\big).
\label{eq:prenorm_ffn}
\end{align}
Here $\mathrm{attn\_out}(\cdot)$ denotes \cref{eq:mhdgca} read as a function of its query, with $\psi(s)$, $\psi_g(g)$ and $\Delta$ held fixed. $\sigma$ is the elementwise sigmoid and $\odot$ the Hadamard product.
$\mathrm{FFN}$ is a two-layer GELU network. The gate vectors $\alpha_\mathrm{attn}, \alpha_\mathrm{ffn} \in \RR^d$ are initialized componentwise to $-5$, giving $\sigma(\alpha) \approx 0.0067$, so each coordinate is modulated independently. We write $\mathrm{DGCA\text{-}Block}_T(\phi(g), \psi(s), \psi_g(g))$ for either output.

The two placements differ on the residual path. Under pre-norm the residual is never normalized, so closing the gates returns $\phi(g)$. Under post-norm it is normalized, so closing them returns a normalized copy instead. \cref{prop:placement} makes this exact; three of the four properties in \cref{sec:theory} rest on it. Gates open during training only where doing so reduces the downstream losses, which yields an automatic curriculum from a fixed goal embedding to a state-conditioned one.

\subsection{Multi-Scale Extension}
\label{sec:multiscale}
A single token count commits the module to one granularity, and the scale at which discrepancies are informative varies across tasks. Navigation often hinges on a coarse global offset, whereas manipulation may hinge on a small object-level mismatch. We therefore apply $L$ independent \DGCA{} blocks at decreasing token counts $T_1 > T_2 > \cdots > T_L$, each with its own projections and attention parameters, and fuse them with learnable weights,
\begin{equation}
\mathrm{level}_\ell = \mathrm{DGCA}\text{-}\mathrm{Block}_{T_\ell}\big(\phi(g), \psi(s), \psi_g(g)\big),
\qquad
\phi_{\textup{\DAGR{}}}(g \mid s) = \sum_{\ell=1}^{L} \mathrm{softmax}(w)_\ell \cdot \mathrm{level}_\ell .
\label{eq:ms_fusion}
\end{equation}
The fusion logits $w \in \RR^L$ start at zero, so every scale contributes equally at step one. We use $L = 3$ with $(T_1, T_2, T_3) = (16, 8, 4)$ throughout. \DAGR{} trains jointly with the downstream algorithm through its value and policy losses, and \cref{app:training} gives the full procedure.
\subsection{What the Refinement Can and Cannot Guarantee}
\label{sec:theory}

Every guarantee below turns on the question raised in \cref{sec:intro}. Does the block return $\phi$ when its gates close? We state that condition, settle which placement meets it, and read off the consequences. We never assume the condition holds. We name it so that each result can declare its dependence on it explicitly.

\begin{assumption}[Closed-Gate Identity]
\label{ass:identity}
For every scale $\ell$, every $s \in \cS$, every $g \in \cG$ and every $\phi$ in the range of the base encoder,
$\mathrm{DGCA\text{-}Block}_{T_\ell}(\phi, \psi(s), \psi_g(g)) \to \phi$ as $\alpha_\mathrm{attn}^{(\ell)}, \alpha_\mathrm{ffn}^{(\ell)} \to -\infty$.
\end{assumption}

\begin{proposition}[Which Placement Satisfies \cref{ass:identity}]
\label{prop:placement}
Pre-norm, given by \cref{eq:prenorm_attn} and \cref{eq:prenorm_ffn}, satisfies \cref{ass:identity}. Post-norm, given by \cref{eq:gated_res_attn} and \cref{eq:gated_res_ffn}, does not. At closed gates the post-norm block equals $\mathrm{LN} \circ \mathrm{LN}$, and $\mathrm{LN}(a\phi + b\mathbf{1}) = \mathrm{LN}(\phi)$ holds for every $a > 0$ and every $b \in \RR$, where $\mathbf{1} \in \RR^d$ is the all-ones vector. The mean and the scale of $\phi(g)$ are therefore erased.
\end{proposition}

One property is independent of the condition. Noise invariance propagates through the computation graph at any gate value, so it survives under either placement, and \cref{thm:noise_invariance} in \cref{app:deferred} states it. The next two are the ones this paper must defend. Throughout we hold $\phi$, $\psi$ and $\psi_g$ fixed and study the map $\phi(g) \mapsto \phi_{\textup{\DAGR{}}}(g \mid s)$ at any fixed $s$.

\begin{theorem}[Sufficiency Preservation]
\label{thm:sufficiency}
Assume $\phi$ is sufficient in the sense of \cref{def:sufficiency}. Write each scale block as $F_\ell(\phi) = \phi + \delta_\ell(\phi, s)$ and suppose the residual $\delta_\ell$ is $c_\ell$-Lipschitz in $\phi$ with $c_\ell < 1$. Then $\phi_{\textup{\DAGR{}}}(\cdot \mid s)$ is sufficient.
\end{theorem}
The placement decides $c_\ell$, not the gate value alone. Under pre-norm the gated branches form the only path from $\phi$ to $\delta_\ell$, so $c_\ell \to 0$ as the gates close, which \cref{lem:gate_control} makes precise. Under post-norm the normalization sits on the residual path and carries no gate. By \cref{prop:placement} that $F_\ell$ is non-injective at closed gates, while $c_\ell < 1$ forces injectivity, so the hypothesis is false there rather than merely unverified. \cref{tab:geometry} measures $c_\ell$ on frozen checkpoints. It reads $0$ for pre-norm at shut gates, as \cref{lem:gate_control} predicts, and exceeds $1$ at the converged gates of both placements. The estimate is local and therefore a lower bound on the true constant, so it can falsify the hypothesis and cannot confirm it. We certify sufficiency in the closed-gate limit of the pre-norm block and nowhere else.

\begin{theorem}[Approximation Error Bound]
\label{thm:approximation_bound}
Let $f : \RR^{d_s} \times \RR^d \to \RR$ be $L_V$-Lipschitz in its second argument, and suppose $\|\phi_{\textup{\DAGR{}}}(g \mid s) - \phi(g)\|_2 \leq B$ holds uniformly over $(s, g) \in \cS \times \cG$. Write $\varepsilon_\mathrm{base} := \sup_{(s, g)} |V^*(s, g) - f(\psi(s), \phi(g))|$ for the error of the base and $\hat V(s, g) := f(\psi(s), \phi_{\textup{\DAGR{}}}(g \mid s))$ for the refined value. Then $\sup_{(s, g)} |V^*(s, g) - \hat V(s, g)| \leq \varepsilon_\mathrm{base} + L_V B$.
\end{theorem}
The bound is gate-controlled only under \cref{ass:identity}. Pre-norm satisfies it, so $B \to 0$ as the gates close. Post-norm does not. At closed gates $B$ attains $B_0 = \sup_g \|\mathrm{LN}(\mathrm{LN}(\phi(g))) - \phi(g)\|_2$, which is of order $\|\phi(g)\|_2$ whenever $\phi(g)$ is unnormalized, as ours is. \cref{tab:geometry} places $B_0$ at $97\%$ of $\|\phi(g)\|_2$. The bound still certifies that refinement cannot blow up, and \cref{tab:sensitivities} shows success falling as the gates open, monotonically on every task but Scene. It does not certify a small surcharge at initialization.

The remaining question is when the refinement is worth its cost. We answer it through a property of the task rather than of the module.

\begin{definition}[$\kappa$-Discrepancy Structure]
\label{def:discrepancy}
Fix a budget $\kappa = (k, L_h)$ with $k \in \mathbb{N}$ and $L_h > 0$. A goal-conditioned task with optimal policy $\pi^*$ has \emph{$\kappa$-discrepancy structure} if there exist an $L_h$-Lipschitz map $h : \cS \times \cS \to \cZ$ with $\cZ \subseteq \RR^k$ and a measurable $\rho : \cZ \to \Delta(\cA)$ such that
\begin{equation}
\pi^*(a \mid s, g) = \rho\big(h(s, g)\big) \quad \text{for all } (s, g) \in \cS \times \cG ,
\label{eq:discrepancy}
\end{equation}
and such that $h$ is not a function of $g$ alone. We call $h$ the \emph{discrepancy map}. A task admitting \cref{eq:discrepancy} at budget $\kappa$ with $h$ a function of $g$ alone is called \emph{$\kappa$-goal-only}.
\end{definition}
The budget is what gives the definition content. Without it the choice $h = \mathrm{id}$ always works and every task qualifies. The informative regime is $k \ll d_s$, where $h$ compresses the pair into a few coordinates, and maze navigation sits there with $k = 2$ once walls and body state are idealized away. The two labels are not complementary, since a task may admit no factorization at that budget at all.

On a $\kappa$-discrepancy task the refinement has a compressed target to expose, and \cref{rem:dagr_helps} in \cref{app:deferred} bounds what exposing it is worth. On a $\kappa$-goal-only task it has none, so the best available outcome is to return the base. Let $\cF_\mathrm{late}$ and $\cF_{\DAGR{}}$ denote the classes of value functions realizable on the inputs $(\psi(s), \phi(g))$ and $(\psi(s), \phi_{\textup{\DAGR{}}}(g \mid s))$. \cref{ass:identity} is what gives $\cF_\mathrm{late} \subseteq \cF_{\DAGR{}}$, and that inclusion is what lets a closed gate attain $\varepsilon_\mathrm{base}$, as \cref{prop:dagr_neutral} in \cref{app:deferred} shows. Our block does not have it.

\begin{corollary}[No Reduction under Post-Norm]
\label{cor:no_reduction}
Let $f^* \in \cF_\mathrm{late}$ attain $\varepsilon_\mathrm{base}$ on a $\kappa$-goal-only task. Take the post-norm block at closed gates. If $f^*(\psi(s), \cdot)$ is not constant on the fibres of $\mathrm{LN} \circ \mathrm{LN}$ over the range of $\phi$, then no $f \in \cF_{\DAGR{}}$ satisfies $f(\psi(s), \phi_{\DAGR{}}(g \mid s)) = f^*(\psi(s), \phi(g))$ for all $(s, g)$, and \cref{prop:dagr_neutral} does not apply.
\end{corollary}

Our base meets the hypothesis. Dual parameterizes the temporal distance as $\psi(s)^\top \phi(g)$, which is homogeneous of degree one in $\phi(g)$. Two goals whose embeddings differ by a positive scale therefore carry different distances while sharing one $\mathrm{LN} \circ \mathrm{LN}$ image. Closing the gates cannot fall back to the base, and \cref{tab:geometry} confirms the mechanism numerically, since the post-norm block retains only $1.8\%$ of $\|\phi(g)\|_2$ at convergence.

\begin{remark}[How Much the Corollary Explains]
\label{rem:opt_gap}
Two variants restore \cref{ass:identity}, namely pre-norm and no normalization at all. On Cube-Double with eight seeds each they reach $43.4$ and $42.3$ against $34.7$ for post-norm, a pooled gain of $8.2$ points at Welch $p = 0.011$, and they are indistinguishable from each other at $p = 0.84$. The recovery therefore tracks the condition and not the placement. It is also partial. Seventeen of the twenty-five points lost against the base value of $60$ remain, the same restoration costs one to six points on each navigation task, and Scene does not move at all. What remains is a gap between the global minimum and the point the optimizer reaches, and \cref{prop:dagr_neutral} constrains only the former.
\end{remark}

\section{Experiments}
\label{sec:experiments}
We evaluate \DAGR{} on OGBench~\citep{park2025ogbench} along three axes. First, whether state-conditioned refinement improves goal reaching across diverse tasks. Second, whether the difference bias contributes beyond standard cross-attention. Third, on which tasks it does not help, and why.

\textbf{Setup.}
\label{sec:setup}
We evaluate on 13 state-based tasks and 7 visual variants. The state-based tasks comprise navigation tasks (PointMaze, AntMaze, HumanoidMaze, and AntSoccer), manipulation tasks (Cube-Single, Cube-Double, and Scene), and discrete-reasoning tasks (Puzzle-3x3 and Puzzle-4x4). The base goal representation is Dual~\citep{park2026dual} with bilinear (inner-product) parameterization, and the downstream offline GCRL algorithm is GCIVL~\citep{park2025ogbench}. \DAGR{} uses $H = 4$ heads with $d_k = 64$, scales $(T_1, T_2, T_3) = (16, 8, 4)$, FFN hidden dimension $256$, and one \DGCA{} block per scale. All gate parameters and difference scalings are initialized to $\alpha_0 = \lambda_0 = -5$, and fusion logits to zero. We use eight seeds for state-based tasks and four seeds for visual tasks. Full hyperparameters are in \cref{tab:hyperparams} of \cref{app:exp_details}. The Orig, VIB, VIP, TRA, BYOL-$\gamma$, and Dual numbers are reproduced from \citet{park2026dual} under their own tuned settings. \DAGR{} uses a single fixed configuration across all tasks rather than per-task tuning. The reported gains are therefore not attributable to task-specific hyperparameter search.

\subsection{Main Results: State-Based Tasks}
\label{sec:main_state}

\begin{table}[t]
\vspace{-3pt}
\centering
\caption{\textbf{Success rates (\%) on state-based OGBench with GCIVL.} Mean $\pm$ std over 8 seeds. The Orig, VIB, VIP, TRA, BYOL-$\gamma$, and Dual columns are reproduced from \citet{park2026dual}, Table~1. \DAGR{} = Dual + MS-\DGCA{} (ours). \textcolor{orange}{\textbf{Orange}} = best, \underline{underline} = second best. The average row is the unweighted mean over the thirteen tasks, with the standard deviation taken across seeds of that mean.}
\label{tab:state_results}
\resizebox{\textwidth}{!}{
\begin{tabular}{l|c|c|c|c|c|c|c}
\toprule
\textbf{Environment} & \textbf{Orig} & \textbf{VIB} & \textbf{VIP} & \textbf{TRA} & \textbf{BYOL-$\gamma$} & \textbf{Dual} & \textbf{\DAGR{}} \\
\midrule
\multicolumn{8}{l}{\textit{Navigation}} \\
pointmaze-medium-navigate  & \underline{78}{\tiny$\pm$8}  & 69{\tiny$\pm$13} & 0{\tiny$\pm$1}   & 3{\tiny$\pm$6}   & 37{\tiny$\pm$7}  & 76{\tiny$\pm$7}  & \textcolor{orange}{\textbf{87}}{\tiny$\pm$8} \\
pointmaze-large-navigate   & \textcolor{orange}{\textbf{52}}{\tiny$\pm$6}  & \underline{50}{\tiny$\pm$7}  & 0{\tiny$\pm$0}   & 1{\tiny$\pm$2}   & 22{\tiny$\pm$12} & 46{\tiny$\pm$6}  & 41{\tiny$\pm$7} \\
antmaze-medium-navigate    & 71{\tiny$\pm$4}  & 68{\tiny$\pm$4}  & 31{\tiny$\pm$5}  & 22{\tiny$\pm$15} & 39{\tiny$\pm$5}  & \underline{75}{\tiny$\pm$4}  & \textcolor{orange}{\textbf{95}}{\tiny$\pm$1} \\
antmaze-large-navigate     & 16{\tiny$\pm$3}  & 9{\tiny$\pm$3}   & 9{\tiny$\pm$2}   & 22{\tiny$\pm$12} & 11{\tiny$\pm$5}  & \underline{28}{\tiny$\pm$11} & \textcolor{orange}{\textbf{82}}{\tiny$\pm$3} \\
antmaze-giant-navigate     & 0{\tiny$\pm$0}   & 0{\tiny$\pm$0}   & 0{\tiny$\pm$0}   & 0{\tiny$\pm$0}   & 0{\tiny$\pm$0}   & 0{\tiny$\pm$0}   & \textcolor{orange}{\textbf{4}}{\tiny$\pm$2} \\
humanoidmaze-medium-navigate & 27{\tiny$\pm$3} & 24{\tiny$\pm$2}  & 7{\tiny$\pm$3}   & 21{\tiny$\pm$3}  & 18{\tiny$\pm$5}  & \underline{29}{\tiny$\pm$3}  & \textcolor{orange}{\textbf{83}}{\tiny$\pm$3} \\
humanoidmaze-large-navigate  & \underline{3}{\tiny$\pm$0}  & \underline{3}{\tiny$\pm$1}   & 1{\tiny$\pm$0}   & 2{\tiny$\pm$1}   & 2{\tiny$\pm$1}   & \underline{3}{\tiny$\pm$2}   & \textcolor{orange}{\textbf{62}}{\tiny$\pm$4} \\
antsoccer-arena-navigate   & \underline{47}{\tiny$\pm$4}  & 34{\tiny$\pm$4}  & 2{\tiny$\pm$1}   & 8{\tiny$\pm$2}   & 11{\tiny$\pm$4}  & 31{\tiny$\pm$3}  & \textcolor{orange}{\textbf{58}}{\tiny$\pm$4} \\
\midrule
\multicolumn{8}{l}{\textit{Manipulation}} \\
cube-single-play           & 52{\tiny$\pm$3}  & \textcolor{orange}{\textbf{90}}{\tiny$\pm$3}  & 40{\tiny$\pm$7}  & 40{\tiny$\pm$5}  & 51{\tiny$\pm$11} & \underline{89}{\tiny$\pm$3}  & 87{\tiny$\pm$2} \\
cube-double-play           & \underline{35}{\tiny$\pm$5}  & 33{\tiny$\pm$3}  & 3{\tiny$\pm$2}   & 7{\tiny$\pm$2}   & 6{\tiny$\pm$4}   & \textcolor{orange}{\textbf{60}}{\tiny$\pm$4}  & \underline{35}{\tiny$\pm$4} \\
scene-play                 & 46{\tiny$\pm$3}  & 58{\tiny$\pm$1}  & 23{\tiny$\pm$6}  & 46{\tiny$\pm$6}  & 44{\tiny$\pm$9}  & \textcolor{orange}{\textbf{72}}{\tiny$\pm$6}  & \underline{59}{\tiny$\pm$2} \\
\midrule
\multicolumn{8}{l}{\textit{Discrete Reasoning}} \\
puzzle-3x3-play            & 5{\tiny$\pm$1}   & \textcolor{orange}{\textbf{14}}{\tiny$\pm$3}  & 3{\tiny$\pm$1}   & 5{\tiny$\pm$1}   & 0{\tiny$\pm$0}   & 5{\tiny$\pm$1}   & 5{\tiny$\pm$1} \\
puzzle-4x4-play            & 14{\tiny$\pm$1}  & 6{\tiny$\pm$3}   & 1{\tiny$\pm$1}   & 10{\tiny$\pm$3}  & 1{\tiny$\pm$2}   & \textcolor{orange}{\textbf{23}}{\tiny$\pm$3}  & \underline{14}{\tiny$\pm$3} \\
\midrule
\textbf{Average}           & 34{\tiny$\pm$1}  & 35{\tiny$\pm$2}  & 9{\tiny$\pm$1}   & 15{\tiny$\pm$2}  & 19{\tiny$\pm$2}  & \underline{41}{\tiny$\pm$2}  & \textcolor{orange}{\textbf{55}}{\tiny$\pm$1} \\
\bottomrule
\end{tabular}
}
\vspace{-0.5cm}
\end{table}
\cref{tab:state_results} reports the state-based results. The pattern is structured. On every navigation task except pointmaze-large, \DAGR{} attains the best or tied-best score, and the largest absolute gains fall where Dual is weak but above zero. Every one of these tasks satisfies \cref{def:discrepancy} in this idealized sense, with the position offset $h(s, g) = \pi_{xy}(s) - \pi_{xy}(g)$, where $\pi_{xy}$ projects a state onto its two planar coordinates. The budget is then $k = 2$, far below the state dimension $d_s$ on the Ant and Humanoid mazes, which is the compressed regime \cref{def:discrepancy} marks as informative. Those carry the largest gains, AntMaze-Giant aside. PointMaze sits at the other end, where the state is already close to the offset and the definition predicts nothing. The outcome there is mixed, an eleven-point gain on the medium maze and a within-noise loss on the large one, as \cref{fig:curves_state} also shows. We read this as consistent with \cref{def:discrepancy} and not as evidence for the bound of \cref{rem:dagr_helps}, whose realizability hypothesis \cref{app:realizability} tests and does not support.

On manipulation the picture is mixed. Cube-Single matches Dual. Cube-Double and Scene fall below it. Neither of the latter two admits a discrepancy map at a small budget. On Cube-Double the optimal action turns on a binary choice of which cube to move first. On Scene it turns on a multi-step ordering decision. Both are combinatorial rather than metric, so no low-dimensional Lipschitz $h(s, g)$ captures them, and state-conditioning has no compressed target to expose.

The gates do stay nearly closed. \cref{fig:gates_cube} keeps every $\sigma(\alpha)$ below $0.013$ for the whole run. Closing them is neither falling back nor staying small. By \cref{cor:no_reduction} the post-norm block at closed gates returns a normalized copy of $\phi(g)$, which erases the scale the bilinear distance reads, so the fallback path of \cref{prop:dagr_neutral} is unavailable and a drop is expected. \cref{tab:geometry} shows that the perturbation is not small either, since the surviving norm is $1.8\%$ of $\|\phi(g)\|_2$ at those same gate values. The size of the drop remains unexplained, and \cref{rem:opt_gap} reports how much restoring the identity recovers. On the state-based puzzles \DAGR{} matches Dual on 3$\times$3 and trails it on 4$\times$4, at $14$ against $23$, where success is low for every late-fusion method. The collapse to zero appears only on the visual puzzles, which \cref{sec:puzzle_analysis} examines.

\subsection{Main Results: Visual Tasks}
\label{sec:main_visual}

\begin{table}[t]
\vspace{-3pt}
\centering
\caption{\textbf{Success rates (\%) on visual OGBench with GCIVL.} Mean $\pm$ std over 4 seeds. The Orig, VIB, VIP, TRA, BYOL-$\gamma$, and Dual columns are reproduced from \citet{park2026dual}, Table~2. The \DAGR{} column is ours, and the bottom average row is our own aggregation. \textcolor{orange}{\textbf{Orange}} = best, \underline{underline} = second best.}
\label{tab:visual_results}
\resizebox{\textwidth}{!}{
\begin{tabular}{l|cccccc|c}
\toprule
\textbf{Environment} & \textbf{Orig} & \textbf{VIB} & \textbf{VIP} & \textbf{TRA} & \textbf{BYOL-$\gamma$} & \textbf{Dual} & \textbf{\DAGR{} (Ours)} \\
\midrule
\multicolumn{8}{l}{\textit{Navigation}} \\
visual-antmaze-medium-navigate & 66{\tiny$\pm$4} & 18{\tiny$\pm$9} & 30{\tiny$\pm$7} & 48{\tiny$\pm$4} & 32{\tiny$\pm$5} & \underline{78}{\tiny$\pm$4} & \textcolor{orange}{\textbf{90}}{\tiny$\pm$2} \\
visual-antmaze-large-navigate  & 26{\tiny$\pm$5} & 5{\tiny$\pm$2}  & 9{\tiny$\pm$1}  & 13{\tiny$\pm$3} & 9{\tiny$\pm$4}  & \underline{40}{\tiny$\pm$4} & \textcolor{orange}{\textbf{52}}{\tiny$\pm$2} \\
\midrule
\multicolumn{8}{l}{\textit{Manipulation}} \\
visual-cube-single-play        & \underline{53}{\tiny$\pm$4} & 18{\tiny$\pm$19} & 39{\tiny$\pm$6} & 31{\tiny$\pm$24} & 35{\tiny$\pm$8} & \textcolor{orange}{\textbf{58}}{\tiny$\pm$5} & 44{\tiny$\pm$2} \\
visual-cube-double-play        & \underline{9}{\tiny$\pm$2}  & 0{\tiny$\pm$0}   & 0{\tiny$\pm$0}  & 3{\tiny$\pm$2}   & 2{\tiny$\pm$1}  & \underline{9}{\tiny$\pm$2}  & \textcolor{orange}{\textbf{11}}{\tiny$\pm$4} \\
visual-scene-play              & 25{\tiny$\pm$2} & 6{\tiny$\pm$3}   & 4{\tiny$\pm$1}  & 15{\tiny$\pm$6}  & 10{\tiny$\pm$8} & \underline{26}{\tiny$\pm$5} & \textcolor{orange}{\textbf{27}}{\tiny$\pm$3} \\
\midrule
\multicolumn{8}{l}{\textit{Discrete Reasoning}} \\
visual-puzzle-3x3-play         & \textcolor{orange}{\textbf{22}}{\tiny$\pm$2} & 0{\tiny$\pm$0}   & 0{\tiny$\pm$0}  & 0{\tiny$\pm$0}   & 0{\tiny$\pm$0}  & 0{\tiny$\pm$0}  & 0{\tiny$\pm$0} \\
visual-puzzle-4x4-play         & \textcolor{orange}{\textbf{65}}{\tiny$\pm$4} & 0{\tiny$\pm$0}   & 0{\tiny$\pm$0}  & 0{\tiny$\pm$0}   & 0{\tiny$\pm$0}  & 0{\tiny$\pm$0}  & 0{\tiny$\pm$0} \\
\midrule
\textbf{Avg.\ (all 7 tasks)}   & \textcolor{orange}{\textbf{38}} & 7 & 12 & 16 & 13 & 30 & \underline{32} \\
\textbf{Avg.\ (excl.\ puzzle)} & 36 & 9 & 16 & 22 & 18 & \underline{42} & \textcolor{orange}{\textbf{45}} \\
\bottomrule
\end{tabular}
\vspace{-0.5cm}
}
\end{table}
\cref{tab:visual_results} reproduces the state-based pattern on navigation. \DAGR{} is highest on both Visual-AntMaze variants, so the directional signal survives the pixel encoder. Visual manipulation splits. Cube-Double and Scene match Dual, whereas Cube-Single falls below it, and \cref{sec:ablation_ca} identifies this as the one task on which state-conditioning helps but our attention rule does not. Visual-Puzzle remains at zero for every late-fusion method, an encoder-level bottleneck we examine in \cref{sec:puzzle_analysis}. Curves are in \cref{fig:curves_visual}.

\subsection{Ablation: Standard Cross-Attention vs.\ \DGCA{}}
\label{sec:ablation_ca}

\begin{table}[t]
\centering
\begin{minipage}[t]{0.55\textwidth}
\centering
\caption{\textbf{Standard cross-attention vs.\ \DGCA{}}, both applied to Dual. Success rate (\%). Dual numbers from \citet{park2026dual}.}
\label{tab:ablation_ca}
\small
\setlength{\tabcolsep}{4pt}
\begin{tabular}{lccc}
\toprule
\textbf{Environment} & \textbf{Dual} & \textbf{+CA} & \textbf{+\DGCA{}} \\
\midrule
\multicolumn{4}{l}{\textit{State-Based}} \\
antmaze-medium       & 75{\tiny$\pm$4}  & \underline{84}{\tiny$\pm$6}  & \textcolor{orange}{\textbf{95}}{\tiny$\pm$1} \\
antmaze-large        & 28{\tiny$\pm$11} & \underline{33}{\tiny$\pm$4}  & \textcolor{orange}{\textbf{82}}{\tiny$\pm$3} \\
humanoidmaze-medium  & 29{\tiny$\pm$3}  & \underline{41}{\tiny$\pm$7}  & \textcolor{orange}{\textbf{83}}{\tiny$\pm$3} \\
humanoidmaze-large   & 3{\tiny$\pm$2}   & \underline{8}{\tiny$\pm$2}   & \textcolor{orange}{\textbf{62}}{\tiny$\pm$4} \\
antsoccer-arena      & 31{\tiny$\pm$3}  & \underline{44}{\tiny$\pm$2}  & \textcolor{orange}{\textbf{58}}{\tiny$\pm$4} \\
cube-double          & \textcolor{orange}{\textbf{60}}{\tiny$\pm$4} & 26{\tiny$\pm$5} & \underline{35}{\tiny$\pm$4} \\
scene                & \textcolor{orange}{\textbf{72}}{\tiny$\pm$6} & 57{\tiny$\pm$1} & \underline{59}{\tiny$\pm$2} \\
\midrule
\multicolumn{4}{l}{\textit{Visual}} \\
visual-antmaze-medium & 78{\tiny$\pm$4}  & \underline{80}{\tiny$\pm$14} & \textcolor{orange}{\textbf{90}}{\tiny$\pm$2} \\
visual-antmaze-large  & 40{\tiny$\pm$4}  & \underline{41}{\tiny$\pm$9}  & \textcolor{orange}{\textbf{52}}{\tiny$\pm$2} \\
visual-cube-single    & \underline{58}{\tiny$\pm$5} & \textcolor{orange}{\textbf{80}}{\tiny$\pm$3} & 44{\tiny$\pm$2} \\
visual-cube-double    & 9{\tiny$\pm$2}   & \underline{10}{\tiny$\pm$2}  & \textcolor{orange}{\textbf{11}}{\tiny$\pm$4} \\
visual-scene          & \underline{26}{\tiny$\pm$5} & 13{\tiny$\pm$2} & \textcolor{orange}{\textbf{27}}{\tiny$\pm$3} \\
visual-puzzle-3x3     & 0{\tiny$\pm$0}   & 0{\tiny$\pm$0}   & 0{\tiny$\pm$0} \\
\bottomrule
\end{tabular}
\end{minipage}
\hfill
\begin{minipage}[t]{0.40\textwidth}
\centering
\caption{\textbf{Component ablation of MS-\DGCA{}}, state-based. The Cube-D reference and normalization rows use eight seeds and every other cell uses four. Removing the gate costs and removing the residual collapses. The difference bias is nearly neutral. The two variants that restore \cref{ass:identity} recover together. All entries come from the sweep of \cref{tab:ablation_components}, which lists the rest. \textcolor{orange}{\textbf{Orange}} = best, \underline{underline} = second best in each column.}
\label{tab:ablation_components_main}
\small
\setlength{\tabcolsep}{2.5pt}
\begin{tabular}{lcc}
\toprule
\textbf{Variant} & \textbf{AntMaze-L} & \textbf{Cube-D} \\
\midrule
\DAGR{} (full)      & \textcolor{orange}{\textbf{85.4}}{\tiny$\pm$2.3} & 34.7{\tiny$\pm$4.0} \\
w/o gate            & 59.8{\tiny$\pm$10.1} & 21.0{\tiny$\pm$3.7} \\
w/o residual        & 11.7{\tiny$\pm$1.7} & 6.1{\tiny$\pm$2.2} \\
w/o FFN             & 58.7{\tiny$\pm$4.1} & 40.7{\tiny$\pm$3.9} \\
w/o diff.\ bias     & \textcolor{orange}{\textbf{85.4}}{\tiny$\pm$1.9} & 33.1{\tiny$\pm$6.2} \\
pre-norm            & 80.1{\tiny$\pm$2.5} & \textcolor{orange}{\textbf{43.4}}{\tiny$\pm$11.4} \\
no norm             & \underline{80.4}{\tiny$\pm$4.8} & \underline{42.3}{\tiny$\pm$7.6} \\
\bottomrule
\end{tabular}
\end{minipage}
\end{table}
\cref{tab:ablation_ca} compares Dual, Dual with a standard cross-attention module (\textbf{+CA}, single-scale, gate at $\sigma(0) = 0.5$, no difference bias), and Dual with the full \DGCA{}. On navigation, \DGCA{} clearly beats \textbf{+CA}, and the gap is widest where Dual is weakest. The two differ in three ways at once, since \DGCA{} adds the difference bias and multi-scale fusion and also closes the gate at initialization. \cref{tab:ablation_components_main,tab:ablation_components} separate them. Removing the difference bias leaves AntMaze-Large at $85.4$, exactly the full model, and a single scale still gives $80.5$. Removing the gate drops it to $59.8$; removing the residual outright drops it to $11.7$, below the Dual base. The gate, not the bias or the scales, is chiefly what \textbf{+CA} is missing. We state this plainly. The mechanism that names our method is not the mechanism that produces our navigation gains.

Visual-Cube-Single is the sharpest counterexample in this paper. \textbf{+CA} reaches $80$ there against $58$ for Dual and $44$ for \DGCA{}. Two readings follow and they point in opposite directions. State-conditioning helps considerably on this task, since the $22$ points that \textbf{+CA} gains over Dual is the largest single-task gain any module produces here, so the failure does not lie on the state-conditional axis. Yet the ingredients separating \DGCA{} from \textbf{+CA} are jointly harmful and together cost $36$ points. Our ablations cannot separate them. One untested account is that single-object visual goal reaching is template matching, for which a similarity-maximizing bias is appropriate and a discrepancy-maximizing one is not.
\subsection{Component Attribution and Scope}
\label{sec:attribution}

Beyond the gate, \cref{tab:ablation_components_main} shows the FFN to be equally load-bearing, since removing it drops AntMaze-Large to $58.7$. Attention selection is not. At $T = 1$ the softmax runs over a single logit and $\Delta \equiv 1$, so selection and bias are both inert by construction, yet that variant still reaches $77.3$ on AntMaze-Large, three points below single-scale $T = 16$ in \cref{tab:sensitivities}. The difference bias is not load-bearing either. \cref{tab:lambda_analysis} reports the learned $\zeta(\lambda)$ on six tasks. It sits at or barely above its initial value on four of them, with Scene and, more mildly, Cube-Single the exceptions, and \cref{fig:gates} shows the gates behaving the same way. The bias therefore acts as an inductive prior at initialization rather than as a converged contribution, and the gated residual, with its gates nearly closed, carries the navigation gain. The encoder rather than $\phi$ sets the scope of that gain, because every late-fusion method scores zero on Visual-Puzzle according to \cref{tab:puzzle_comparison}. The IMPALA CNN pools away the state-goal correspondence before $\phi(g)$ is formed, so a post-encoder module can only reweight what the encoder preserves.

\section{Conclusion}
\label{sec:conclusion}
\DAGR{} conditions a static goal embedding on the current state through gated multi-scale cross-attention. Three of its four guarantees require the closed-gate identity. A pre-norm residual satisfies it. Our post-norm residual does not. Navigation improves wherever a low-dimensional discrepancy map exists. Our ablations trace this gain to the gated residual, not the difference bias. The manipulation regression is partly architectural. A closed post-norm gate erases the scale that the bilinear distance reads. Restoring the identity recovers $8$ of the $25$ points lost on Cube-Double. Our tokens are not object-aware. The visual puzzle failure sits in the encoder. We plan to pair \DAGR{} with an object-aware or early-fusion encoder.

\bibliographystyle{plainnat}
\bibliography{ref}

\setcounter{tocdepth}{-1}

\clearpage
\appendix
\crefalias{section}{appendix}
\crefalias{subsection}{appendix}
\crefalias{subsubsection}{appendix}
\setlength{\parindent}{0pt}
\renewcommand{\contentsname}{Contents of Appendix}
\renewcommand{\thesection}{\Alph{section}}

\addtocontents{toc}{\protect\setcounter{tocdepth}{3}} 

\begingroup
\hypersetup{linkcolor=black}
\begin{spacing}{1.5}
\tableofcontents 
\end{spacing}
\endgroup

\clearpage

\section{Training Procedure}
\label{app:training}

\DAGR{} is trained jointly with the downstream offline GCRL algorithm. The base representation loss $\cL_{\textup{rep}}$ produces $\phi(g)$ without cross-attention, exactly as in the underlying Dual method, so the temporal-distance structure of $\phi$ relied on by \cref{thm:sufficiency,thm:noise_invariance} is not perturbed. Only the value, Q, and policy losses propagate gradients through MS-\DGCA{}. Its parameters comprise the per-scale token projections $\{W_s^{(\ell)}, W_g^{(\ell)}\}$ (\cref{eq:tokens}), attention matrices $\{W_h^{Q,(\ell)}, W_h^{K,(\ell)}, W_h^{V,(\ell)}, W^{O,(\ell)}\}$ (\cref{eq:mhdgca}), gate vectors $\{\alpha_{\textup{attn}}^{(\ell)}, \alpha_{\textup{ffn}}^{(\ell)}\}$ and per-head difference scalars $\{\lambda_h^{(\ell)}\}$ (\cref{eq:gated_res_attn,eq:gated_res_ffn,eq:dgca}), per-scale FFN parameters, and fusion logits $w \in \RR^L$ (\cref{eq:ms_fusion}). Each call $\mathrm{MS\text{-}\DGCA{}}(\phi(g), \psi(s), \psi_g(g))$ proceeds in four steps. Step (i) projects $\psi(s), \psi_g(g)$ into pseudo-tokens via \cref{eq:tokens}. Step (ii) computes the difference map via \cref{eq:diff_map}. Step (iii) runs multi-head \DGCA{} and the gated residual block via \cref{eq:mhdgca,eq:gated_res_attn,eq:gated_res_ffn} at each of $L$ scales. Step (iv) fuses the scale outputs via \cref{eq:ms_fusion}. At step zero the gates are near-closed, so $\tilde\phi$ equals $\phi(g)$ up to the normalization of \cref{prop:placement}. Gates open only when doing so reduces the downstream losses, which realizes the late-fusion-to-refinement curriculum discussed in \cref{sec:dgca} and bounded by \cref{thm:approximation_bound}.

\begin{algorithm}[h]
\caption{\DAGR{} layered on top of the Dual goal representation~\citep{park2026dual} with GCIVL~\citep{kostrikov2022offline,park2025ogbench} as the downstream offline GCRL algorithm. Lines highlighted in {\color{orange}orange} mark the differences from a plain Dual + GCIVL baseline.}
\label{alg:dagr}
\begin{algorithmic}[1]
\REQUIRE Dataset $\cD$, encoders $\phi, \psi, \psi_g$, MS-\DGCA{} module, value $V_{\eta_V}$, Q-network $Q_{\eta_Q}$, policy $\pi_\theta$, AWR temperature $\beta$, expectile $\tau$, discount $\gamma$, target rate $\tau_{\textup{tgt}}$.
\STATE {\color{orange}Initialize $\alpha_{\textup{attn}}^{(\ell)}, \alpha_{\textup{ffn}}^{(\ell)} \leftarrow -5$ (\cref{eq:gated_res_attn,eq:gated_res_ffn}), $\lambda_h^{(\ell)} \leftarrow -5$ (\cref{eq:dgca}), $w \leftarrow \mathbf{0}$ (\cref{eq:ms_fusion}).}
\FOR{each gradient step}
    \STATE Sample $(s, a, s') \sim \cD$ and future goal $g$ (goal-sampling distribution in \cref{tab:hyperparams}).
    \STATE Update $\phi$ via $\cL_{\textup{rep}}$ on $\phi(g)$ alone, with MS-\DGCA{} \emph{not} in the graph.
    \STATE {\color{orange}$\tilde\phi \leftarrow \mathrm{MS\text{-}\DGCA{}}(\phi(g), \psi(s), \psi_g(g))$ and $\tilde\phi' \leftarrow \mathrm{MS\text{-}\DGCA{}}(\phi(g), \psi(s'), \psi_g(g))$ \hfill\textit{// $\tilde\phi \approx \phi_{\textup{\DAGR{}}}(g\mid s)$}}
    \STATE $y_V \leftarrow r(s, g) + \gamma \bar V(\psi(s'), \tilde\phi')$ using EMA target $\bar V$.
    \STATE Update $V_{\eta_V}$ on $\cL_V = L_\tau\big(Q(\psi(s), \tilde\phi, a) - V(\psi(s), \tilde\phi)\big)$ with $L_\tau(u) = |\tau - \mathbbm{1}[u<0]|\,u^2$~\citep{kostrikov2022offline}.
    \STATE Update $Q_{\eta_Q}$ on $\cL_Q = \big(Q(\psi(s), \tilde\phi, a) - y_V\big)^2$.
    \STATE $A \leftarrow Q(\psi(s), \tilde\phi, a) - V(\psi(s), \tilde\phi)$ (stop-gradient).
    \STATE Update $\pi_\theta$ on $\cL_\pi = -\exp(A / \beta)\,\log \pi_\theta(a \mid \psi(s), \tilde\phi)$.
    \STATE Soft-update target: $\bar V \leftarrow (1 - \tau_{\textup{tgt}}) \bar V + \tau_{\textup{tgt}} V_{\eta_V}$.
\ENDFOR
\RETURN $\pi_\theta, \phi, \psi, \psi_g, \mathrm{MS\text{-}\DGCA{}}$.
\end{algorithmic}
\end{algorithm}
\clearpage
\section{Notation}
\label{app:notation}
\begin{table}[h]
\centering
\caption{Notation used in the proofs and not introduced in the main text.}
\label{tab:notation}
\begin{tabular}{p{0.30\linewidth} p{0.62\linewidth}}
\toprule
\textbf{Symbol} & \textbf{Description} \\
\midrule
\multicolumn{2}{l}{\textit{General notation}} \\
$\phi(g\mid s)$ & Deterministic state-parameterized goal encoding; the bar denotes parameter dependence, not a conditional distribution. Equivalently $\phi(g;s)$ or $\phi_s(g)$. \\
$\varphi^\vee(g),\ \varphi_{\textup{\DAGR{}}}^\vee(g\mid s)$ & Dual and \DAGR{} goal functionals from $\cS$ to $\RR$ (\cref{eq:dual_discrete,eq:dagr_discrete}); the latter is $\Delta_{s,g}$-weighted version of the former. \\
$d^*(s,g)$ & Optimal temporal distance $\log_\gamma V^*(s,g)$ (\cref{sec:prelim}); equals shortest-path length in deterministic environments. \\
$\Delta_{s,g}(\cdot)$ & State-dependent weighting $\cS\to[0,1]$ in the \DAGR{} functional (\cref{eq:dagr_discrete}); approximated by the per-token difference map $\Delta_t$ of \cref{eq:diff_map}. \\
\addlinespace
\multicolumn{2}{l}{\textit{Discrepancy structure (\cref{def:discrepancy})}} \\
$h(s,g),\ \rho$ & Discrepancy map $\cS\times\cS\to\cZ$ and corresponding policy factor $\cZ\to\Delta(\cA)$ such that $\pi^*(a\mid s,g)=\rho(h(s,g))$. \\
$\cZ$ & Latent discrepancy space; concrete instances include $\RR^k$ (e.g., $h(s,g)=\pi_{xy}(s)-\pi_{xy}(g)$ for maze navigation). \\
``$\kappa$-goal-only'' & Task satisfying \cref{eq:discrepancy} at budget $\kappa$ with $h$ a function of $g$ alone, that is, $\pi^*(a\mid s,g)=\rho(\bar h(g))$ for some $\bar h$. \\
\addlinespace
\multicolumn{2}{l}{\textit{Information-theoretic quantities (\cref{rem:info_bottleneck,rem:late_fusion_failure,rem:dagr_helps})}} \\
$S, G, A^*$ & Random variables for state, goal, and optimal action. \\
$I(X;Y\mid Z)$ & Conditional mutual information between $X$ and $Y$ given $Z$. \\
$X\to Y\to Z$ & Markov chain, equivalently $X\indep Z \mid Y$. \\
$\mathfrak{R}_n(\cdot)$ & Rademacher complexity of a function class on $n$ samples. $U_n^{\textup{late}}$ and $U_n^{\DAGR{}}$ denote the upper bounds on it under late-fusion and \DAGR{} inputs, as in \cref{eq:rad_gap}. \\
$L_g$ & Lipschitz constant of a function $g$. \\
$L_{\psi^{-1}},\ L_{\phi^{-1}}$ & Lipschitz constants of implicit encoder inverses required by late-fusion to recover $(s,g)$ (\cref{rem:dagr_helps}). \\
\addlinespace
\multicolumn{2}{l}{\textit{Sufficiency and noise invariance (\cref{thm:sufficiency,thm:noise_invariance})}} \\
$F_\ell,\ \delta_\ell$ & Per-scale block map $F_\ell(\phi) = \phi + \delta_\ell(\phi, s)$ and its residual perturbation. \\
$c_\ell$ & Lipschitz constant of the residual $\delta_\ell$; $c_\ell < 1$ is assumed in \cref{thm:sufficiency} and gate-bounded only under pre-norm (\cref{lem:gate_control}). \\
$\mathrm{Recover}_s$ & Inverse of the map $\phi(g)\mapsto\phi_{\textup{\DAGR{}}}(g\mid s)$ on its image. \\
$o_g = (x_g,\epsilon_g)$ & Goal observation split into a task-relevant component $x_g$ and an exogenous noise component $\epsilon_g\in\cE$. \\
\addlinespace
\multicolumn{2}{l}{\textit{Approximation and conditional improvement (\cref{thm:approximation_bound,rem:dagr_helps,prop:dagr_neutral})}} \\
$\hat V(s,g)$ & Approximated value $\hat V(s,g) = f(\psi(s),\phi_{\textup{\DAGR{}}}(g\mid s))$. \\
$\varepsilon_\mathrm{base},\ B,\ L_V$ & Base approximation error, perturbation radius $\|\phi_{\textup{\DAGR{}}}(g\mid s)-\phi(g)\|_2\leq B$, and Lipschitz constant of $f$ in its second argument. \\
$\cF_\mathrm{late},\ \cF_{\DAGR{}}$ & Function classes of the two encoders. $\cF_\mathrm{late}\subseteq\cF_{\DAGR{}}$ holds under \cref{ass:identity}; \cref{cor:no_reduction} shows it fails for post-norm. Both are defined in \cref{sec:theory}.\\
\bottomrule
\end{tabular}
\end{table}
\clearpage
\section{Deferred Statements}
\label{app:deferred}

This appendix holds the results and the background that \cref{sec:problem,sec:theory} defer. Proofs are in \cref{app:proofs}.

\begin{remark}[Information Bottleneck of State-Independent Goal Representations]
\label{rem:info_bottleneck}
For any state-independent $\phi$ we have $I(A^*; \phi(G) \mid S) \leq I(A^*; G \mid S)$, with equality if and only if $\phi(G)$ retains all goal information relevant to $A^*$ given $S$.
\end{remark}

\begin{remark}[Late Fusion Admits a Looser Complexity Bound on Relational Tasks]
\label{rem:late_fusion_failure}
Suppose $\pi^*(a \mid s, g) = f(s - g)$. Realizing $\pi^*$ from the inputs $(\psi(s), \phi(g))$ requires composing $f$ with measurable inverse selections of $\psi$ and of $\phi$. Writing $\mathfrak{R}_n$ for the Rademacher complexity of a function class on $n$ samples, its composition property~\citep{bartlett2002rademacher} then yields an upper bound on the policy class that exceeds the bound available under direct access to $s - g$ by a factor $L_{\psi^{-1}} L_{\phi^{-1}} \geq 1$, where $L_{\psi^{-1}}$ and $L_{\phi^{-1}}$ are the Lipschitz constants of those inverse selections.
\end{remark}

\begin{theorem}[Noise Invariance Preservation]
\label{thm:noise_invariance}
Suppose goal observations decompose as $o_g = (x_g, \epsilon_g)$ with $\epsilon_g$ task-irrelevant, and suppose the base encoders satisfy $\phi(o_g) = \phi(x_g)$ and $\psi_g(o_g) = \psi_g(x_g)$ for all $(x_g, \epsilon_g)$. Then $\phi_{\textup{\DAGR{}}}(o_g \mid s) = \phi_{\textup{\DAGR{}}}(x_g \mid s)$ for every $\epsilon_g$ and every $s \in \cS$.
\end{theorem}
This is the one result in \cref{sec:theory} that does not rest on \cref{ass:identity}. Its content is confined to the visual pipeline. On state-based tasks $\psi_g$ is the identity and the goal observation carries no exogenous component, so the hypothesis is vacuous there and the gains of \cref{tab:state_results} owe nothing to it.

\begin{remark}[Upper Bounds on Discrepancy-Structured Tasks]
\label{rem:dagr_helps}
Suppose the task has $\kappa$-discrepancy structure with map $h$, and suppose the pair $(\psi, \phi_{\textup{\DAGR{}}})$ realizes $h$ through an $L_h$-Lipschitz $\tilde h$ while $(\psi, \phi)$ admits no $L_h$-Lipschitz realization. Let $\pi^* = \rho \circ h$ with $\rho$ being $L_\rho$-Lipschitz, and let $\cH_\rho$ be the policy class realizing $\rho$. The Rademacher composition bound~\citep{bartlett2002rademacher} then yields upper bounds $U_n^{\textup{late}}$ and $U_n^{\DAGR{}}$ on $\mathfrak{R}_n$ of the two policy classes,
\begin{equation}
U_n^{\DAGR{}} \leq L_h L_\rho \mathfrak{R}_n(\cH_\rho),
\qquad
U_n^{\textup{late}} \leq L_{\psi^{-1}} L_{\phi^{-1}} L_h L_\rho \mathfrak{R}_n(\cH_\rho) .
\label{eq:rad_gap}
\end{equation}
\end{remark}
We state this as a remark for two reasons. It compares two upper bounds and does not show that the realized class is harder under late fusion. Its realizability hypothesis concerns the trained encoders rather than the architecture, and \cref{app:realizability} tests that hypothesis and does not support it. \cref{sec:main_state} therefore does not rely on it.

\begin{proposition}[Reduction to Base on Goal-Only Tasks]
\label{prop:dagr_neutral}
Let the task be $\kappa$-goal-only, so that $\pi^*(a \mid s, g) = \rho(\bar h(g))$ for some $\bar h$ depending on $g$ alone. Let $\phi$ be sufficient in the sense of \cref{def:sufficiency} for $\pi^*$ with $\rho \circ \bar h \in \cF_{\textup{late}}$, and assume \cref{ass:identity}. Then the global minimum of any value or policy loss over $\cF_{\textup{\DAGR{}}}$ that achieves $\varepsilon_{\textup{base}}$ is attained as $\alpha_\mathrm{attn}, \alpha_\mathrm{ffn} \to -\infty$, equivalently at $B = 0$ in \cref{thm:approximation_bound} and at $\phi_{\textup{\DAGR{}}}(g \mid s) \equiv \phi(g)$.
\end{proposition}
\section{On the Realizability Assumption of \cref{rem:dagr_helps}}
\label{app:realizability}

\cref{rem:dagr_helps} assumes a realizability gap. There exists an $L_h$-Lipschitz $\tilde h$ with $\tilde h(\psi(s), \phi_{\textup{\DAGR{}}}(g \mid s)) = h(s, g)$, and no $L_h$-Lipschitz function on $(\psi(s), \phi(g))$ realizes $h$. We do not verify this, and we are explicit about what our experiments can and cannot say.

The gains of \cref{sec:main_state} are consistent with the assumption. They are equally consistent with added module capacity and with a better-conditioned optimization path, and we can rule out neither. \cref{sec:ablation_components} is what constrains which part of the architecture does the work.

The assumption is directly testable, and we test it. We freeze a trained state-based agent, draw $(s,g)$ pairs, and fit two probes on $(\psi(s), \phi(g))$ and on $(\psi(s), \phi_{\DAGR{}}(g \mid s))$, matched in capacity and spectrally normalized to $\sigma_{\max} = 1$, against the target $h(s,g) = \pi_{xy}(s) - \pi_{xy}(g)$. The gap runs the other way. Held-out MSE is $0.14 \pm 0.08$ against $52.2 \pm 0.03$ on antmaze-large, and $0.25 \pm 0.04$ against $37.2 \pm 0.04$ on humanoidmaze-large.

We do not claim the assumption holds, and we do not read the probe as refuting it either. Three limits apply. The budget is pinned at $L_h = 1$, whereas the assumption quantifies over $L_h$. The two inputs sit on different scales, since $\phi(g)$ is an unnormalized network output while $\phi_{\DAGR{}}(g \mid s)$ is an $\mathrm{LN} \circ \mathrm{LN}$ image by \cref{prop:placement}, and a $1$-Lipschitz probe cannot exceed its input range. And $\psi$ is the identity here, so $L_{\psi^{-1}} = 1$ and late fusion is handed $\pi_{xy}(s)$ for free. The result is therefore evidence for \cref{prop:placement} rather than a verdict on \cref{rem:dagr_helps}. Settling it requires standardized goal-side inputs, a swept budget, and a pre-norm arm.

\section{Theoretical Proofs}
\label{app:proofs}

This appendix contains complete proofs for all results stated in \cref{sec:method}. We restate each result before proving it. All notation follows \cref{tab:notation}.

\subsection{Proof of \cref{rem:info_bottleneck}}
\label{app:proof_info_bottleneck}
\begin{proof}
Since $\phi$ does not depend on $S$, conditioning on $S = s$ makes $A^*$ a function of $\pi^*(\cdot \mid s, g)$, which depends on $g$ only through $\phi(g)$. That is the Markov chain $G \to \phi(G) \to A^*$ at $S = s$. The data processing inequality~\citep{cover1999elements} gives $I(A^*; \phi(G) \mid S = s) \leq I(A^*; G \mid S = s)$, and taking the expectation over $S$ yields the claim. Equality holds exactly when $A^* \indep G \mid (\phi(G), S)$, that is, when $\phi(G)$ is a sufficient statistic for $G$ with respect to $A^*$ given $S$.
\end{proof}

\subsection{Proof of \cref{rem:late_fusion_failure}}
\label{app:proof_late_fusion_failure}
\begin{proof}
Under late fusion, $\pi_\theta(a \mid s, g) = h_\theta(\psi(s), \phi(g))$ with fixed encoders $\psi$ and $\phi$. To realize $\pi^* = f \circ \mathrm{subtract}$ from these inputs, $h_\theta$ must invert $\psi$ to recover $s$, invert $\phi$ to recover $g$, compute $s - g$, and apply $f$. The composition property of Rademacher complexity~\citep{bartlett2002rademacher} gives $\mathfrak{R}_n(g_1 \circ g_2) \leq L_{g_2} \cdot \mathfrak{R}_n(g_1)$ for $L_{g_2}$-Lipschitz $g_2$. Adding the inversion layers introduces extra Lipschitz factors $L_{\psi^{-1}}, L_{\phi^{-1}}$ in front of $\mathfrak{R}_n(f)$, increasing the bound on the effective complexity by a multiplicative factor relative to the case where $s - g$ is directly available. A representation that exposes $s - g$ to the policy therefore admits a strictly smaller upper bound on sample complexity for realizing the same $\pi^*$.
\end{proof}

\subsection{Proof of \cref{prop:placement}}
\label{app:proof_placement}

\begin{proof}
Write $\mathrm{LN}(\phi) = \gamma \odot (\phi - \mu(\phi)\mathbf{1}) / \sqrt{v(\phi) + \varepsilon} + \beta$, with $\mu(\phi)$ the coordinate mean of $\phi$ and $v(\phi)$ its coordinate variance. Fix $a > 0$ and $b \in \RR$. Then $\mu(a\phi + b\mathbf{1}) = a\mu(\phi) + b$ and $v(a\phi + b\mathbf{1}) = a^2 v(\phi)$, so
\[
\mathrm{LN}(a\phi + b\mathbf{1}) \;=\; \gamma \odot \frac{a\,\big(\phi - \mu(\phi)\mathbf{1}\big)}{\sqrt{a^2 v(\phi) + \varepsilon}} + \beta \;=\; \mathrm{LN}(\phi),
\]
exactly at $\varepsilon = 0$ and up to the regularizer otherwise. The composition $\mathrm{LN} \circ \mathrm{LN}$ inherits the invariance.

Now send $\alpha_\mathrm{attn}, \alpha_\mathrm{ffn} \to -\infty$. In \cref{eq:gated_res_attn,eq:gated_res_ffn} both gated terms vanish, so the block returns $\mathrm{out} = \mathrm{LN}(\mathrm{LN}(\phi))$. That map is constant on every set $\{a\phi + b\mathbf{1} : a > 0,\, b \in \RR\}$. It is therefore not injective, hence not the identity, and the two coordinates it discards are precisely the mean and the scale of $\phi$. In \cref{eq:prenorm_attn,eq:prenorm_ffn} the same limit gives $\mathbf{x} = \phi$ and then $\mathrm{out} = \phi$, because the residual path carries $\phi$ untouched.
\end{proof}

\subsection{Proof of \cref{thm:sufficiency}}
\label{app:proof_sufficiency}

We first establish a key lemma.

\begin{lemma}[Injectivity of the Gated Residual Mapping]
\label{lem:info_preserve}
Under the assumptions of \cref{thm:sufficiency}, for any fixed $s \in \cS$ and any scale $\ell$, the map $\phi(g) \mapsto \mathrm{\DGCA{}\text{-}Block}_{T_\ell}(\phi(g), \psi(s), \psi_g(g))$ is injective. Consequently, $\phi(g)$ is recoverable from the pair $(s, \phi_{\textup{\DAGR{}}}(g \mid s))$.
\end{lemma}

\begin{proof}
Define $\delta_\ell(\phi, s) := F_\ell(\phi) - \phi$, so $F_\ell(\phi) = \phi + \delta_\ell(\phi, s)$ holds for either placement by construction. The theorem assumes $\|\delta_\ell(\phi_1, s) - \delta_\ell(\phi_2, s)\| \leq c_\ell \|\phi_1 - \phi_2\|$ with $c_\ell < 1$. Hence
\[
\|F_\ell(\phi_1) - F_\ell(\phi_2)\| \;\geq\; \|\phi_1 - \phi_2\| - \|\delta_\ell(\phi_1, s) - \delta_\ell(\phi_2, s)\| \;\geq\; (1 - c_\ell)\, \|\phi_1 - \phi_2\| \;>\; 0
\]
whenever $\phi_1 \neq \phi_2$. Therefore $F_\ell$ is injective.

For the fusion of \cref{eq:ms_fusion}, the same bound gives $\langle F_\ell(\phi_1) - F_\ell(\phi_2), \phi_1 - \phi_2 \rangle \geq (1 - c_\ell)\|\phi_1 - \phi_2\|^2 > 0$ whenever $\phi_1 \neq \phi_2$. A convex combination with strictly positive softmax weights preserves a positive inner product with $\phi_1 - \phi_2$, so it cannot vanish, and $\phi_{\textup{\DAGR{}}}(g_1 \mid s) = \phi_{\textup{\DAGR{}}}(g_2 \mid s)$ forces $\phi(g_1) = \phi(g_2)$. The inverse is the iteration $\phi^{(k+1)} = \phi_{\textup{\DAGR{}}}(g \mid s) - \sum_\ell \mathrm{softmax}(w)_\ell\, \delta_\ell(\phi^{(k)}, s)$, which converges by the Banach fixed-point theorem because $\max_\ell c_\ell < 1$.
\end{proof}

\begin{proof}
By the assumed sufficiency of $\phi$, there exists $\pi^* : \cS \times \RR^d \to \Delta(\cA)$ with $V^{\pi^*(\cdot \mid s, \phi(g))}(s, g) = V^*(s, g)$. By \cref{lem:info_preserve}, the map $\phi(g) \mapsto \phi_{\textup{\DAGR{}}}(g \mid s)$ is invertible on its image with inverse $\mathrm{Recover}_s$. Define $\tilde \pi(a \mid s, z) := \pi^*(a \mid s, \mathrm{Recover}_s(z))$. Then for any trajectory generated by $\tilde \pi(\cdot \mid s_t, \phi_{\textup{\DAGR{}}}(g \mid s_t))$, at each step
\[
\tilde \pi(\cdot \mid s_t, \phi_{\textup{\DAGR{}}}(g \mid s_t)) \;=\; \pi^*(\cdot \mid s_t, \mathrm{Recover}_{s_t}(\phi_{\textup{\DAGR{}}}(g \mid s_t))) \;=\; \pi^*(\cdot \mid s_t, \phi(g)).
\]
The value attained by $\tilde \pi$ therefore matches $V^*(s, g)$.
\end{proof}

\begin{lemma}[Gate Control of the Residual under Pre-Norm]
\label{lem:gate_control}
Let the block be pre-norm \eqref{eq:prenorm_attn}--\eqref{eq:prenorm_ffn}. Let $\phi \mapsto \mathrm{attn\_out}(\mathrm{LN}(\phi))$ be $L_\mathrm{attn}$-Lipschitz and $\mathbf{x} \mapsto \mathrm{FFN}(\mathrm{LN}(\mathbf{x}))$ be $L_\mathrm{ffn}$-Lipschitz. Put $a := \|\sigma(\alpha_\mathrm{attn})\|_\infty L_\mathrm{attn}$ and $b := \|\sigma(\alpha_\mathrm{ffn})\|_\infty L_\mathrm{ffn}$. Then the residual of \cref{thm:sufficiency} obeys $c_\ell \leq a + b\,(1 + a)$, so $c_\ell \to 0$ as $\alpha_\mathrm{attn}, \alpha_\mathrm{ffn} \to -\infty$.
\end{lemma}

\begin{proof}
For any vector $u$ and any gate $\alpha$, $\|\sigma(\alpha) \odot u\|_2 \leq \|\sigma(\alpha)\|_\infty \|u\|_2$. By \cref{eq:prenorm_attn}, $\mathbf{x} = \phi + \sigma(\alpha_\mathrm{attn}) \odot \mathrm{attn\_out}(\mathrm{LN}(\phi))$, so $\phi \mapsto \mathbf{x}$ is $(1 + a)$-Lipschitz. By \cref{eq:prenorm_ffn},
\[
\delta_\ell(\phi, s) \;=\; \mathrm{out} - \phi \;=\; \sigma(\alpha_\mathrm{attn}) \odot \mathrm{attn\_out}\big(\mathrm{LN}(\phi)\big) \;+\; \sigma(\alpha_\mathrm{ffn}) \odot \mathrm{FFN}\big(\mathrm{LN}(\mathbf{x})\big).
\]
The first summand is $a$-Lipschitz in $\phi$. The second is $b$-Lipschitz in $\mathbf{x}$, hence $b(1 + a)$-Lipschitz in $\phi$. Summing gives the bound, and both $a$ and $b$ carry a factor $\sigma(\alpha) \to 0$.
\end{proof}

\begin{remark}[No Analogue under Post-Norm]
Under \cref{eq:gated_res_attn,eq:gated_res_ffn} the normalization lies on the residual path and is multiplied by no gate, so $\delta_\ell$ does not vanish with $\sigma(\alpha)$. \cref{prop:placement} shows $F_\ell$ is non-injective at closed gates, while $c_\ell < 1$ forces injectivity by \cref{lem:info_preserve}. No choice of $L_\mathrm{attn}$ and $L_\mathrm{ffn}$ repairs that.
\end{remark}

\subsection{Proof of \cref{thm:noise_invariance}}
\label{app:proof_noise}
\begin{proof}
Write $o_g = (x_g, \epsilon_g)$. By assumption $\phi(o_g) = \phi(x_g)$ and $\psi_g(o_g) = \psi_g(x_g)$, while $\psi(s)$ does not depend on $o_g$. For the Dual base this follows from \citet[Theorem~3.2]{park2026dual}, so layering \DAGR{} on Dual requires no further argument. Every node of the MS-\DGCA{} graph is a deterministic function of those three quantities. The tokens of \cref{eq:tokens}, the difference map of \cref{eq:diff_map}, the queries, keys and values of \cref{eq:mhdgca}, the normalizations, the gates and the FFN therefore inherit independence of $\epsilon_g$ in turn, under either placement and at every gate value. The fusion weights are parameters. Hence $\phi_{\textup{\DAGR{}}}(o_g \mid s) = \phi_{\textup{\DAGR{}}}(x_g \mid s)$.
\end{proof}
\begin{remark}[Scope]
\cref{thm:noise_invariance} is the one result in \cref{sec:theory} that does not rest on \cref{ass:identity}. Its content is confined to the visual pipeline. On state-based tasks $\psi_g$ is the identity and the goal observation carries no exogenous component, so the hypothesis is vacuous and the gains of \cref{tab:state_results} owe nothing to it.
\end{remark}

\subsection{Proof of \cref{thm:approximation_bound}}
\label{app:proof_approx}

\begin{proof}
Fix $(s, g)$ and write $V_\phi(s, g) = f(\psi(s), \phi(g))$. By the triangle inequality,
\[
|V^*(s, g) - \hat V(s, g)| \;\leq\; |V^*(s, g) - V_\phi(s, g)| + |V_\phi(s, g) - \hat V(s, g)|.
\]
The first term is bounded by $\varepsilon_\mathrm{base}$ by definition. For the second term, the Lipschitz assumption gives $|f(\psi(s), \phi(g)) - f(\psi(s), \phi_{\textup{\DAGR{}}}(g \mid s))| \leq L_V \, \|\phi(g) - \phi_{\textup{\DAGR{}}}(g \mid s)\|_2 \leq L_V \cdot B$. Taking the supremum over $(s, g)$ yields the claim.
\end{proof}

\begin{remark}[Behaviour at Initialization Depends on the Placement]
Under pre-norm, $\phi_{\textup{\DAGR{}}}(g \mid s) - \phi(g) = \sum_\ell \mathrm{softmax}(w)_\ell \, \delta_\ell(\phi, s)$ is a convex combination of gated branch outputs, so \cref{lem:gate_control} gives $B_0 \leq \|\sigma(\alpha_0)\|_\infty (M_\mathrm{attn} + M_\mathrm{ffn})$, with $M_\mathrm{attn}, M_\mathrm{ffn}$ the sup-norms of the two branch outputs over the range of $\phi$. At $\sigma(\alpha_0) = \sigma(-5) \approx 0.007$ this is negligible, and $B$ grows only as the optimizer opens the gates.

Under post-norm the same limit gives $\phi_{\textup{\DAGR{}}}(g \mid s) = \mathrm{LN}(\mathrm{LN}(\phi(g)))$ by \cref{prop:placement}, so $B_0 = \sup_{g} \|\mathrm{LN}(\mathrm{LN}(\phi(g))) - \phi(g)\|_2$. This quantity does not shrink with the gates. It is of order $\|\phi(g)\|_2$ whenever $\phi(g)$ sits away from the fixed set of $\mathrm{LN} \circ \mathrm{LN}$. Our $\phi$ is the goal branch of a bilinear temporal-distance network with no output normalization, so $B_0$ is not small in any run we report, and $L_V B_0$ is not a negligible surcharge.
\end{remark}

\subsection{Proof of \cref{rem:dagr_helps}}
\label{app:proof_dagr_helps}

\begin{proof}
By the discrepancy structure assumption, $\pi^*(a \mid s, g) = \rho(h(s, g))$ where $\rho$ is $L_\rho$-Lipschitz and $h$ admits an $L_h$-Lipschitz realization. Any policy realizing $\pi^*$ on the respective encoder inputs computes $\rho \circ h$ as a function of those inputs.

In the \DAGR{} case, $\tilde h(\psi(s), \phi_{\textup{\DAGR{}}}(g \mid s)) = h(s, g)$ with $\tilde h$ being $L_h$-Lipschitz by assumption, so the policy network realizes $\rho \circ \tilde h$. The composition property of Rademacher complexity~\citep{bartlett2002rademacher} gives
\[
U_n^{\DAGR{}} \;\leq\; L_h \cdot L_\rho \cdot \mathfrak{R}_n(\cH_\rho).
\]

In the late-fusion case, $(\psi, \phi)$ does not admit any $L_h$-Lipschitz realization of $h$ by assumption. Hence any $h^{\textup{late}} : \RR^{d_s} \times \RR^d \to \cZ$ with $h^{\textup{late}}(\psi(s), \phi(g)) = h(s, g)$ must factor through measurable inverse selections $\psi^{-1}, \phi^{-1}$, yielding $h^{\textup{late}} = h \circ (\psi^{-1} \times \phi^{-1})$. The Lipschitz constant of $h^{\textup{late}}$ is at most $L_h \max(L_{\psi^{-1}}, L_{\phi^{-1}})$. Both factors are at least one, because $\psi$ and $\phi$ reduce dimension in any realistic offline GCRL encoder, so their inverse selections cannot contract. The composition bound then gives
\[
U_n^{\textup{late}} \;\leq\; L_{\psi^{-1}} \cdot L_{\phi^{-1}} \cdot L_h \cdot L_\rho \cdot \mathfrak{R}_n(\cH_\rho).
\]
The ratio $U_n^{\textup{late}} / U_n^{\DAGR{}} = L_{\psi^{-1}} \cdot L_{\phi^{-1}} \geq 1$, with strict inequality whenever $\psi$ or $\phi$ strictly reduces dimension, which is the case for every encoder considered in this paper.
\end{proof}

\subsection{Proof of \cref{prop:dagr_neutral}}
\label{app:proof_dagr_neutral}

\begin{proof}
By assumption there exists $f^* \in \cF_{\textup{late}}$ with $f^*(\psi(s), \phi(g)) = \rho(\bar h(g))$ and downstream value error $\varepsilon_{\textup{base}}$. We first show $f^*$ is realizable in $\cF_{\textup{\DAGR{}}}$. Send $\alpha_{\textup{attn}}, \alpha_{\textup{ffn}} \to -\infty$ componentwise. By \cref{ass:identity} each \DGCA{}-Block then returns $\mathrm{out} = \phi(g)$, and the fusion of \cref{eq:ms_fusion} is a convex combination of identical terms, so $\phi_{\textup{\DAGR{}}}(g \mid s) = \phi(g)$ for all $(s, g)$. Hence $f^*(\psi(s), \phi_{\textup{\DAGR{}}}(g \mid s)) = f^*(\psi(s), \phi(g))$ attains $\varepsilon_{\textup{base}}$, which is $B = 0$ in \cref{thm:approximation_bound}. This step is the only one that uses \cref{ass:identity}, and \cref{cor:no_reduction} records what happens without it.

Now consider any \DAGR{} configuration with $B > 0$, that is, $\phi_{\textup{\DAGR{}}}(g \mid s) = \phi(g) + \delta(s, g)$ with $\|\delta\|_2 > 0$ for some $(s, g)$. By the goal-only assumption, $V^*(s, g) = f^*(\psi(s), \phi(g))$ depends on $g$ only through $\phi(g)$. For any $f \in \cF_{\textup{\DAGR{}}}$ that is $L_V$-Lipschitz in its second argument:
\[
|V^*(s, g) - f(\psi(s), \phi(g) + \delta(s, g))| \;\geq\; |V^*(s, g) - f(\psi(s), \phi(g))| - L_V \|\delta(s, g)\|_2.
\]
The first term on the right is at least $\varepsilon_{\textup{base}}$ in the worst case over $(s, g)$. Therefore any minimizer attaining error exactly $\varepsilon_{\textup{base}}$ must satisfy $L_V \|\delta(s, g)\|_2 = 0$ uniformly, that is, $\delta \equiv 0$ and $B = 0$.
\end{proof}
\begin{proof}[Proof of \cref{cor:no_reduction}]
At closed gates $\phi_{\textup{\DAGR{}}}(g \mid s) = \mathrm{LN}(\mathrm{LN}(\phi(g)))$ by \cref{prop:placement}, a quantity that does not depend on $s$. Suppose some $f$ satisfied $f(\psi(s), \mathrm{LN}(\mathrm{LN}(\phi(g)))) = f^*(\psi(s), \phi(g))$ for all $(s, g)$. By hypothesis there are $g_1, g_2$ and an $s$ with $\mathrm{LN}(\mathrm{LN}(\phi(g_1))) = \mathrm{LN}(\mathrm{LN}(\phi(g_2)))$ and $f^*(\psi(s), \phi(g_1)) \neq f^*(\psi(s), \phi(g_2))$. Then $f$ would take two values at one argument, a contradiction.

The Dual base meets the hypothesis. Let two goals share a fibre, that is, $\phi(g_2) = a\phi(g_1) + b\mathbf{1}$ with $a > 0$ and $(a, b) \neq (1, 0)$. The temporal distance is $\psi(s)^\top \phi(g)$, so
\[
\psi(s)^\top \phi(g_2) - \psi(s)^\top \phi(g_1) \;=\; (a - 1)\,\psi(s)^\top \phi(g_1) + b\,\psi(s)^\top \mathbf{1},
\]
which vanishes only on a measure-zero set of states. Any $f^*$ realizing $V^*$ must therefore separate $g_1$ from $g_2$, and is not constant on that fibre.
\end{proof}
\begin{table}[h]
\vspace{-3pt}
\centering
\caption{Hyperparameters for all experiments.}
\label{tab:hyperparams}
\begin{tabular}{lc}
\toprule
\textbf{Hyperparameter} & \textbf{Value} \\
\midrule
\multicolumn{2}{l}{\textit{\DAGR{} Module (MS-\DGCA{})}} \\
Scale levels $L$ & 3 \\
Token counts $(T_1, T_2, T_3)$ & $(16, 8, 4)$ \\
Attention heads $H$ per level & 4 \\
Head dimension $d_k$ & 64 \\
Model dimension $d_m = H \cdot d_k$ & 256 \\
FFN hidden dimension & 256 \\
Gate initialization $\alpha_0$ & $-5$ \\
Difference scaling initialization $\lambda_0$ & $-5$ \\
Fusion logit initialization $w_0$ & $0$ \\
\DGCA{} blocks per scale & 1 \\
\midrule
\multicolumn{2}{l}{\textit{Dual Representation}} \\
Representation type & bilinear (inner product) \\
Goal representation dimension & 256 \\
Representation hidden dims & $(512, 512, 512)$ \\
Representation expectile (state) & 0.9 \\
Representation expectile (visual) & 0.7 \\
\midrule
\multicolumn{2}{l}{\textit{GCIVL Training}} \\
Learning rate & $3 \times 10^{-4}$ \\
Optimizer & Adam \\
Batch size (state / visual) & 1024 / 256 \\
Discount $\gamma$ & 0.99 \\
Target network update rate $\tau$ & 0.005 \\
Value expectile & 0.9 \\
AWR temperature $\beta$ & 10.0 \\
Training steps (state / visual) & $10^6$ / $5 \times 10^5$ \\
Seeds (state / visual) & 8 / 4 \\
\midrule
\multicolumn{2}{l}{\textit{Visual Encoder}} \\
Architecture & IMPALA-small \\
Stack sizes & $(16, 32, 32)$ \\
Residual blocks per stack & 1 \\
MLP hidden dims & $(512,)$ \\
Layer normalization & True \\
Image augmentation probability & 0.5 \\
\midrule
\multicolumn{2}{l}{\textit{Goal Sampling}} \\
Value: current state $p_{\textup{cur}}$ & 0.2 \\
Value: geometric future $p_{\textup{geom}}$ & 0.5 \\
Value: random $p_{\textup{rand}}$ & 0.3 \\
Actor: trajectory future $p_{\textup{traj}}$ & 1.0 \\
\midrule
\multicolumn{2}{l}{\textit{Value and Policy Networks}} \\
Hidden dimensions & $(512, 512, 512)$ \\
Activation & GELU \\
Layer normalization & True \\
\bottomrule
\end{tabular}
\vspace{-0.5cm}
\end{table}
\section{Experimental Details}
\label{app:exp_details}

\subsection{Environment Details}

We follow the experimental setup of \citet{park2026dual} on OGBench~\citep{park2025ogbench}. The locomotion suite uses PointMaze, AntMaze, and HumanoidMaze with medium, large, and giant variants of increasing layout complexity and path length. The manipulation suite uses a 6-DoF UR5e robot arm on Cube (Single requires placing one cube at a target, Double requires coordinating two), Scene (multi-object interaction with a cube, drawer, window, and button locks, where evaluation tasks chain up to eight atomic behaviors), and Puzzle (a Lights-Out variant in which the $3 \times 3$ grid admits $2^9 = 512$ button states and the $4 \times 4$ grid admits $2^{16} = 65{,}536$). Visual variants render the same tasks as $64 \times 64$ RGB images, with the arm made transparent and colors adjusted for full observability.

\subsection{Hyperparameters}
\cref{tab:hyperparams} lists all hyperparameters used for the \DAGR{} module, the Dual base representation, GCIVL training, the visual encoder, goal sampling, and the value and policy networks across all experiments reported in this paper.
\subsection{Network Architecture Details}

\textbf{Base representation network.}
The Dual representation uses a GCBilinearValue network with separate state ($\psi$) and goal ($\phi$) encoders, each consisting of a 3-layer MLP with hidden dimensions $(512, 512, 512)$, GELU activations, and layer normalization. The temporal distance is parameterized as $d(s, g) = \psi(s)^\top \phi(g)$ with output dimension $256$.

\textbf{MS-DGCA module.}
The module receives three inputs, namely the goal representation $\phi(g) \in \RR^{256}$, the state encoder output $\psi(s) \in \RR^{512}$ for visual tasks (or $\RR^{d_s}$ for state-based tasks), and the goal encoder output $\psi_g(g) \in \RR^{512}$. At each of the three scale levels, $\psi(s)$ and $\psi_g(g)$ are independently projected into $T_\ell$ pseudo-tokens of dimension $d_m = 256$ via separate learned linear projections. The difference map, attention, gated residual, and FFN then proceed as in \cref{sec:dgca}. The three scale outputs are combined through softmax-weighted fusion as in \cref{eq:ms_fusion}.

\textbf{Downstream networks.}
The value and actor are standard MLPs with hidden dimensions $(512, 512, 512)$. They receive the concatenation of $\psi(s)$ and the \DAGR{}-enhanced goal representation $\phi(g \mid s)$. The value network uses an ensemble of two heads. The actor outputs a Gaussian with constant standard deviation.

\subsection{Computational Resources and Overhead Measurements}

All experiments were conducted on NVIDIA A100 GPUs. State-based experiments require approximately three hours per seed and visual experiments approximately five hours per seed. Empirical overhead measurements taken on a single A100 are summarized below: per-step training time on the state-based Cube-Double agent increases from $17.1$ ms (Dual) to $31.4$ ms (\DAGR{}), and per-step training time on the visual Cube-Double agent increases from $47.6$ ms to $73.8$ ms. Inference time per call increases from $1.6$ ms to $4.0$ ms (state-based) and from $4.1$ ms to $5.9$ ms (visual). The overhead is non-trivial in relative terms but the absolute training budget remains practical.

\section{Additional Experimental Results}
\label{app:additional}
\subsection{Component Ablation}
\label{sec:ablation_components}

\providecommand{\projrng}[2]{\textcolor{gray}{\textit{[#1,\,#2]}}}
\begin{table}[t]
\centering
\begin{minipage}[t]{0.95\textwidth}
\centering
\caption{\textbf{Component ablation.} Four seeds per cell, all from the current sweep and all trained for $10^6$ steps; the Cube-D entries of the reference, pre-norm and no-norm rows use eight. \textcolor{orange}{\textbf{Orange}} = best, \underline{underline} = second best among measured entries in each column, with ties at the displayed precision sharing a mark. The reference configuration recurs in five sweeps, and its group means span $0.6$ points on Cube-D, $1.3$ on HumM-M and $0.7$ on Scene, but $6.8$ on AntM-L ($78.6$ to $85.4$ over twenty runs). Differences below the relevant span are not resolved. Avg.\ is the unweighted mean of the six displayed means (four navigation and two manipulation tasks); it mixes eight-seed and four-seed entries and is descriptive only.}
\label{tab:ablation_components}
\small
\setlength{\tabcolsep}{2.5pt}
\resizebox{\linewidth}{!}{%
\begin{tabular}{lccccccc}
\toprule
\textbf{Variant} & \textbf{AntM-L} & \textbf{AntM-M} & \textbf{Cube-D} & \textbf{HumM-L} & \textbf{HumM-M} & \textbf{Scene} & \textbf{Avg.} \\
\midrule
\DAGR{} (full, post-norm) & \textcolor{orange}{\textbf{85.4}}{\tiny$\pm$2.3} & \textcolor{orange}{\textbf{95.5}}{\tiny$\pm$0.8} & 34.7{\tiny$\pm$4.0} & \textcolor{orange}{\textbf{65.4}}{\tiny$\pm$2.5} & \textcolor{orange}{\textbf{81.8}}{\tiny$\pm$2.2} & \underline{58.9}{\tiny$\pm$0.1} & \textcolor{orange}{\textbf{70.3}} \\
\midrule
w/o gate ($\sigma(\alpha)\!\equiv\!1$) & 59.8{\tiny$\pm$10.1} & 89.4{\tiny$\pm$5.5} & 21.0{\tiny$\pm$3.7} & 7.6{\tiny$\pm$3.3} & 45.9{\tiny$\pm$4.6} & 56.0{\tiny$\pm$1.1} & 46.6 \\
w/o residual & 11.7{\tiny$\pm$1.7} & 19.6{\tiny$\pm$3.2} & 6.1{\tiny$\pm$2.2} & 2.6{\tiny$\pm$0.8} & 4.4{\tiny$\pm$1.2} & 36.0{\tiny$\pm$9.0} & 13.4 \\
w/o diff.\ bias ($\zeta(\lambda)\!=\!0$) & \textcolor{orange}{\textbf{85.4}}{\tiny$\pm$1.9} & 92.1{\tiny$\pm$2.9} & 33.1{\tiny$\pm$6.2} & \underline{62.7}{\tiny$\pm$2.7} & 79.8{\tiny$\pm$2.4} & 58.1{\tiny$\pm$2.6} & 68.5 \\
w/o FFN & 58.7{\tiny$\pm$4.1} & 87.3{\tiny$\pm$5.4} & 40.7{\tiny$\pm$3.9} & 6.8{\tiny$\pm$2.7} & 43.2{\tiny$\pm$8.9} & 57.5{\tiny$\pm$4.8} & 49.0 \\
w/o multi-scale ($T\!=\!16$) & 80.5{\tiny$\pm$4.5} & 91.6{\tiny$\pm$3.8} & 39.1{\tiny$\pm$5.4} & 55.2{\tiny$\pm$8.2} & \underline{81.2}{\tiny$\pm$1.7} & \textcolor{orange}{\textbf{59.5}}{\tiny$\pm$1.1} & 67.9 \\
2 blocks per scale & \underline{82.2}{\tiny$\pm$3.5} & 93.7{\tiny$\pm$1.9} & 35.6{\tiny$\pm$7.1} & 60.4{\tiny$\pm$5.3} & 80.6{\tiny$\pm$1.3} & 58.6{\tiny$\pm$1.2} & 68.5 \\
\midrule
pre-norm & 80.1{\tiny$\pm$2.5} & \underline{94.2}{\tiny$\pm$0.7} & \textcolor{orange}{\textbf{43.4}}{\tiny$\pm$11.4} & 61.0{\tiny$\pm$2.9} & 76.6{\tiny$\pm$2.0} & 58.3{\tiny$\pm$1.4} & \underline{68.9} \\
no norm & 80.4{\tiny$\pm$4.8} & 89.3{\tiny$\pm$2.1} & \underline{42.3}{\tiny$\pm$7.6} & 60.2{\tiny$\pm$2.6} & 78.1{\tiny$\pm$1.7} & 58.7{\tiny$\pm$0.8} & 68.2 \\
\bottomrule
\end{tabular}%
}
\end{minipage}
\hfill
\begin{minipage}[t]{0.8\textwidth}
\centering
\caption{\textbf{Sensitivity sweeps.} Four seeds per cell, each trained for $10^6$ steps. $T$ is the token count of a single-scale block, whereas the default fuses $T \in \{16, 8, 4\}$. $H$ is swept at matched width $d_m = 256$, because $d_m = H d_k$ otherwise confounds the sweep with capacity. $H = 4$ and $\alpha_0 = -5$ are the default configuration. \textcolor{orange}{\textbf{Orange}} = best, \underline{underline} = second best among measured entries within each block and column.}
\label{tab:sensitivities}
\small
\setlength{\tabcolsep}{3pt}
\begin{tabular}{llcccc}
\toprule
\textbf{Sweep} & \textbf{Value} & \textbf{AntM-L} & \textbf{Cube-D} & \textbf{HumM-M} & \textbf{Scene} \\
\midrule
\multirow{5}{*}{$T$}
 & 1  & 77.3{\tiny$\pm$2.3} & 38.2{\tiny$\pm$1.7} & \underline{82.5}{\tiny$\pm$1.8} & \textcolor{orange}{\textbf{64.7}}{\tiny$\pm$4.5} \\
 & 4  & 78.3{\tiny$\pm$1.6} & \underline{39.4}{\tiny$\pm$2.1} & 82.0{\tiny$\pm$1.8} & \underline{61.5}{\tiny$\pm$0.9} \\
 & 8  & \underline{79.4}{\tiny$\pm$2.4} & \textcolor{orange}{\textbf{40.5}}{\tiny$\pm$2.8} & \textcolor{orange}{\textbf{83.2}}{\tiny$\pm$1.4} & 60.3{\tiny$\pm$2.0} \\
 & 16 & \textcolor{orange}{\textbf{80.5}}{\tiny$\pm$4.5} & 39.1{\tiny$\pm$5.4} & 81.2{\tiny$\pm$1.7} & 59.5{\tiny$\pm$1.1} \\
 & 32 & 71.3{\tiny$\pm$3.3} & 37.3{\tiny$\pm$3.3} & 82.3{\tiny$\pm$1.9} & 58.6{\tiny$\pm$1.1} \\
\midrule
\multirow{3}{*}{$H$}
 & 1 & 82.2{\tiny$\pm$4.6} & \textcolor{orange}{\textbf{38.0}}{\tiny$\pm$2.3} & 81.9{\tiny$\pm$2.1} & \textcolor{orange}{\textbf{66.1}}{\tiny$\pm$8.7} \\
 & 2 & \textcolor{orange}{\textbf{84.9}}{\tiny$\pm$2.9} & \underline{35.2}{\tiny$\pm$2.2} & \textcolor{orange}{\textbf{82.9}}{\tiny$\pm$1.3} & \underline{61.4}{\tiny$\pm$3.2} \\
 & 4 & \underline{84.4}{\tiny$\pm$2.2} & 34.2{\tiny$\pm$3.1} & \underline{82.3}{\tiny$\pm$1.4} & 58.6{\tiny$\pm$4.7} \\
\midrule
\multirow{3}{*}{$\alpha_0$}
 & $-5$ & \textcolor{orange}{\textbf{84.6}}{\tiny$\pm$4.0} & \textcolor{orange}{\textbf{34.1}}{\tiny$\pm$4.4} & \textcolor{orange}{\textbf{82.0}}{\tiny$\pm$2.1} & \underline{58.4}{\tiny$\pm$1.4} \\
 & $-2$ & \underline{83.4}{\tiny$\pm$0.9} & \underline{30.5}{\tiny$\pm$1.4} & \underline{75.6}{\tiny$\pm$2.6} & \textcolor{orange}{\textbf{58.8}}{\tiny$\pm$1.2} \\
 & $0$  & 67.8{\tiny$\pm$6.6} & 28.3{\tiny$\pm$1.8} & 54.8{\tiny$\pm$3.4} & 56.8{\tiny$\pm$4.6} \\
\bottomrule
\end{tabular}
\end{minipage}
\vspace{-0.5cm}
\end{table}
\cref{tab:ablation_components} reports the complete component ablation. The main-text \cref{tab:ablation_components_main} is the subset that isolates the gain. Three rows appear only here. Single-scale $T = 16$ trails the reference by five points on AntMaze-Large and gains four on Cube-Double; only the Cube-Double gap clears the spread in the caption, but \cref{tab:sensitivities} shows every $T$ behaves that way, so the ordering is consistent where the AntMaze-Large gap alone is not. Two blocks per scale add nothing. The no-norm row, also in the main text, belongs with pre-norm rather than with the reference: both restore \cref{ass:identity}, both recover on Cube-Double, and both pay on navigation. \cref{rem:opt_gap} reads the pair.

\subsection{Sensitivity Analyses}
\label{sec:sensitivity}
\cref{tab:sensitivities} sweeps three parameters. Every single-scale $T$ trails the multi-scale default on AntMaze-Large and exceeds it on Cube-Double, five cases out of five on each task, so the fusion helps navigation and mildly hurts manipulation. Within the sweep AntMaze-Large moves by three points from $T = 1$, where selection and bias are inert, to $T = 16$, and drops only at $T = 32$. Across $H$ at matched width, navigation varies by less than three points, whereas Cube-Double and Scene both decrease monotonically. We retain $H = 4$, since the configuration is held fixed across all twenty tasks; tuning it on the manipulation tasks where the module underperforms would be the wrong trade. The gate initialization is the sharpest of the three. Opening it costs $17$ points on AntMaze-Large and $27$ on HumanoidMaze-Medium, monotonically, which is what \cref{thm:approximation_bound} allows and part of what separates \DGCA{} from \textbf{+CA} in \cref{tab:ablation_ca}.

\subsection{Where the Module Looks: Attention and Gate Analysis}
\label{sec:lambda_analysis}
\begin{figure*}[t]
\centering
\begin{subfigure}[t]{0.48\textwidth}
    \centering
    \includegraphics[width=\textwidth]{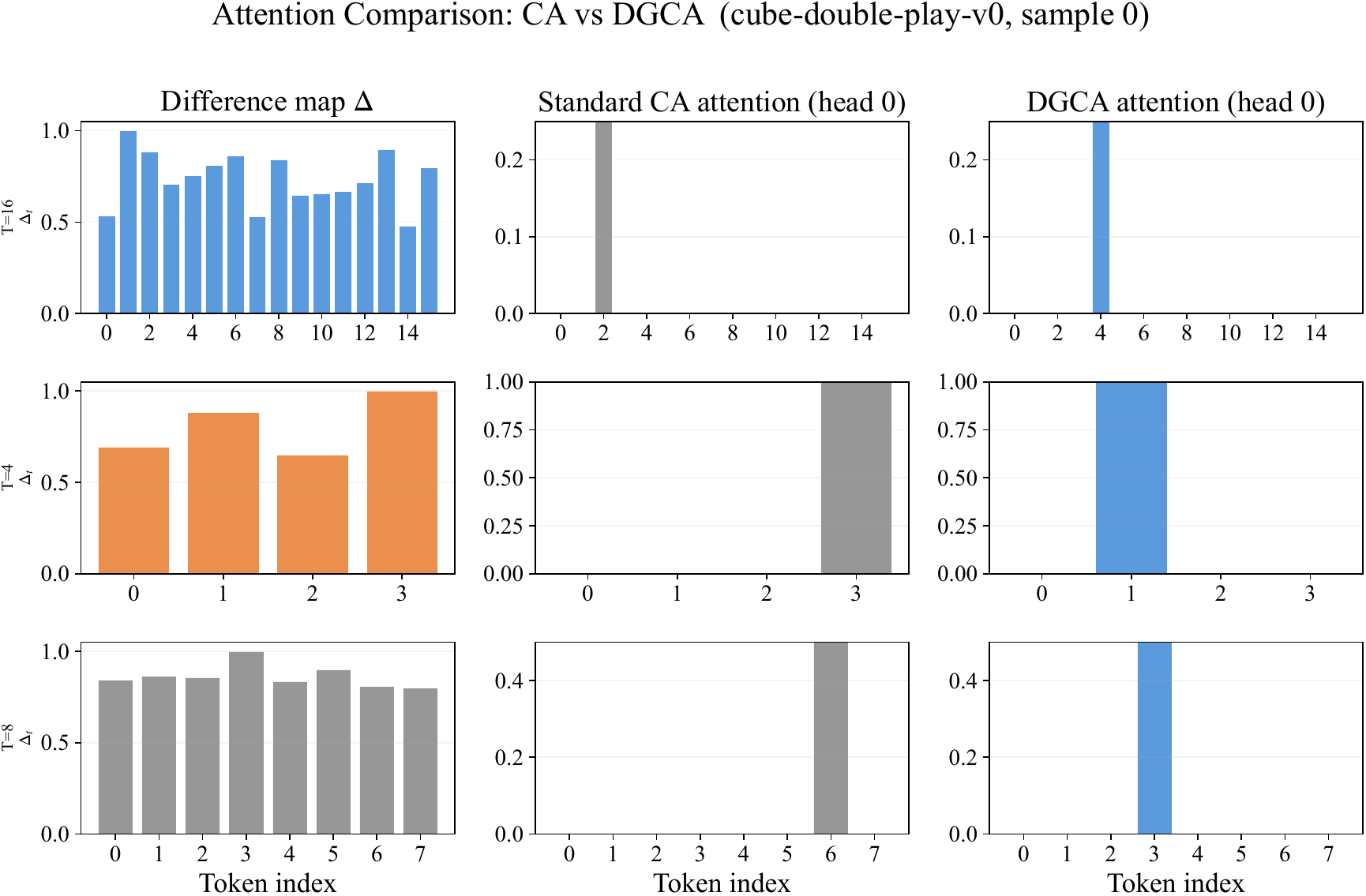}
    \caption{Sample 0}
\end{subfigure}
\hfill
\begin{subfigure}[t]{0.48\textwidth}
    \centering
    \includegraphics[width=\textwidth]{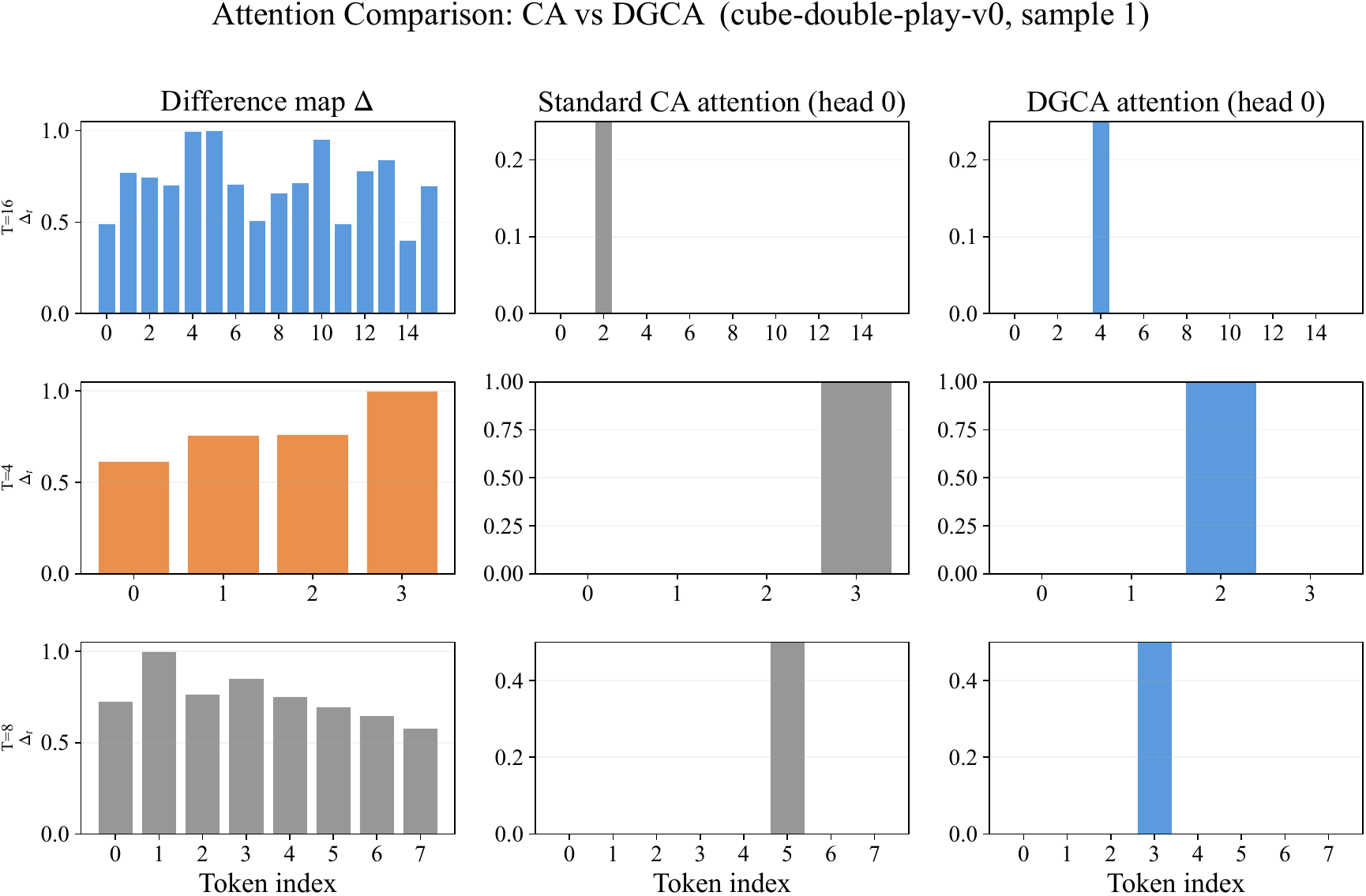}
    \caption{Sample 1}
\end{subfigure}

\vspace{4pt}

\begin{subfigure}[t]{0.48\textwidth}
    \centering
    \includegraphics[width=\textwidth]{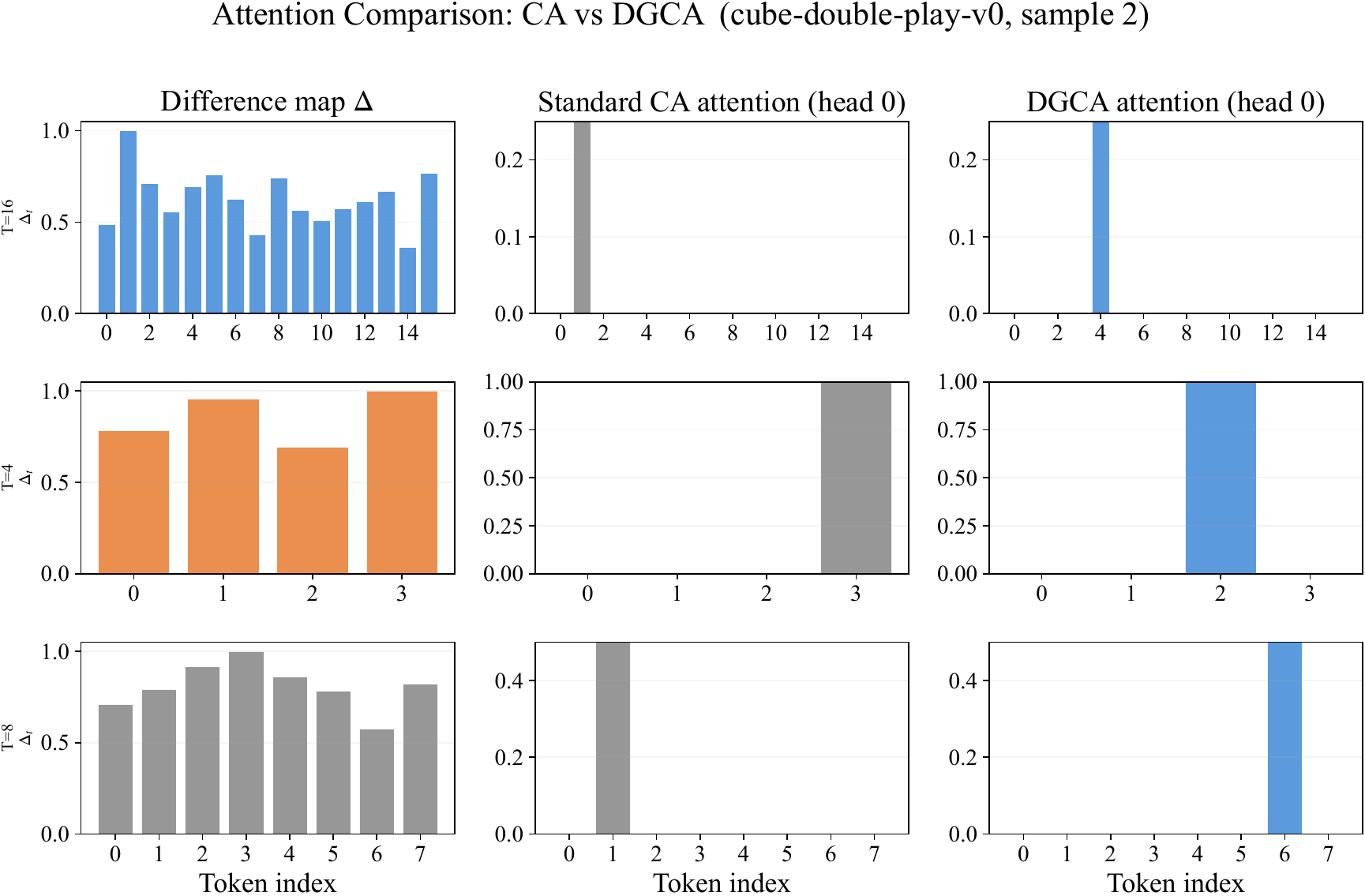}
    \caption{Sample 2}
\end{subfigure}
\hfill
\begin{subfigure}[t]{0.48\textwidth}
    \centering
    \includegraphics[width=\textwidth]{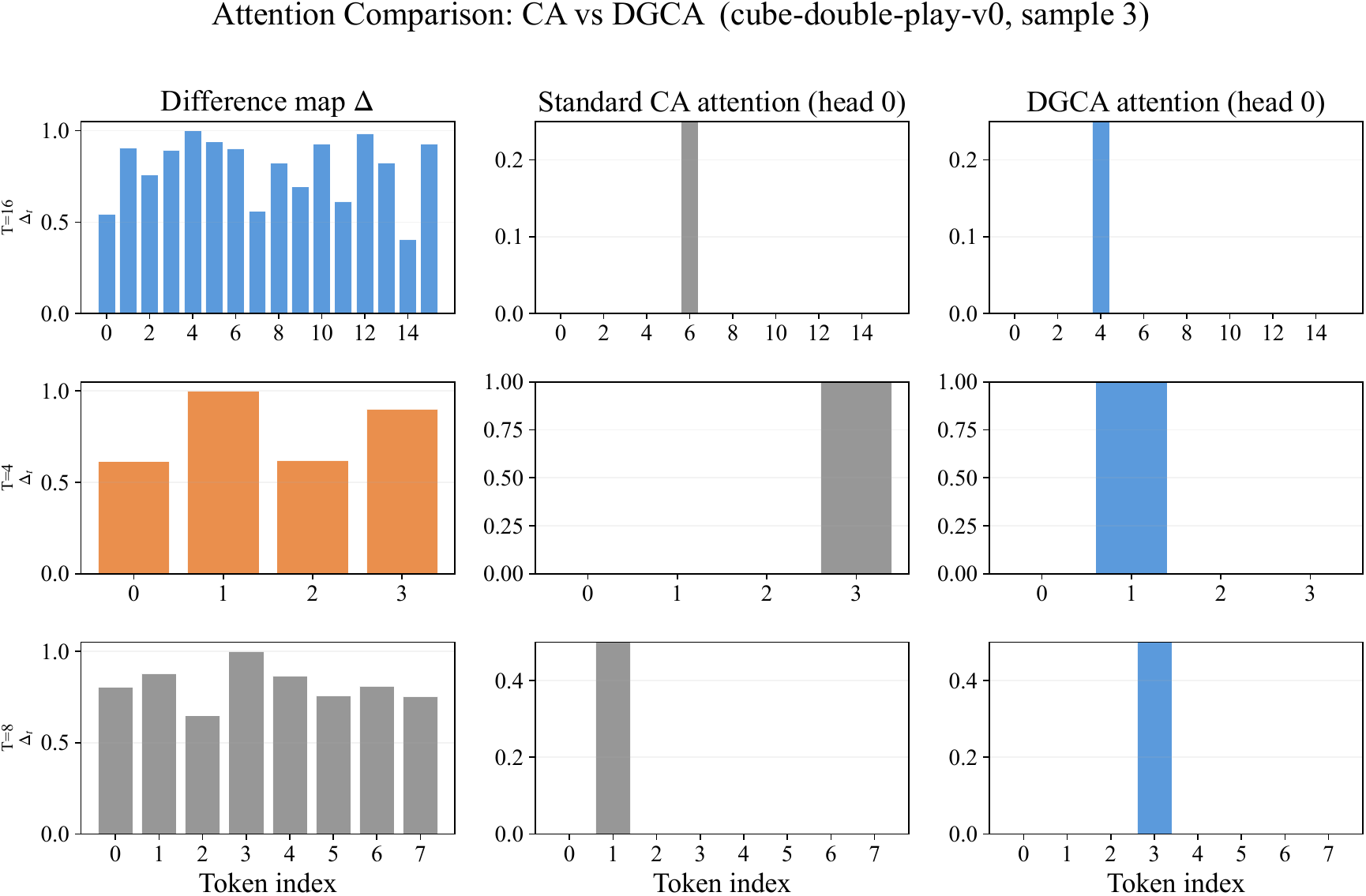}
    \caption{Sample 3}
\end{subfigure}
\caption{\textbf{Attention versus difference map on four Cube-Double samples.} Each panel shows three rows (token counts $T \in \{16, 4, 8\}$), with $\Delta$ in the first column (left bars), standard CA attention in the middle, and \DGCA{} attention on the right. Both CA and \DGCA{} place near-all attention mass on a single token. \DGCA{} chooses a different token than CA on every sample, but neither systematically aligns with the argmax of $\Delta$. On this manipulation task, attention concentration is sharp, but the choice of token is not driven by the difference map alone.}
\label{fig:attention}
\vspace{-0.5cm}
\end{figure*}

\begin{figure}[t]
\centering
\includegraphics[width=0.98\textwidth]{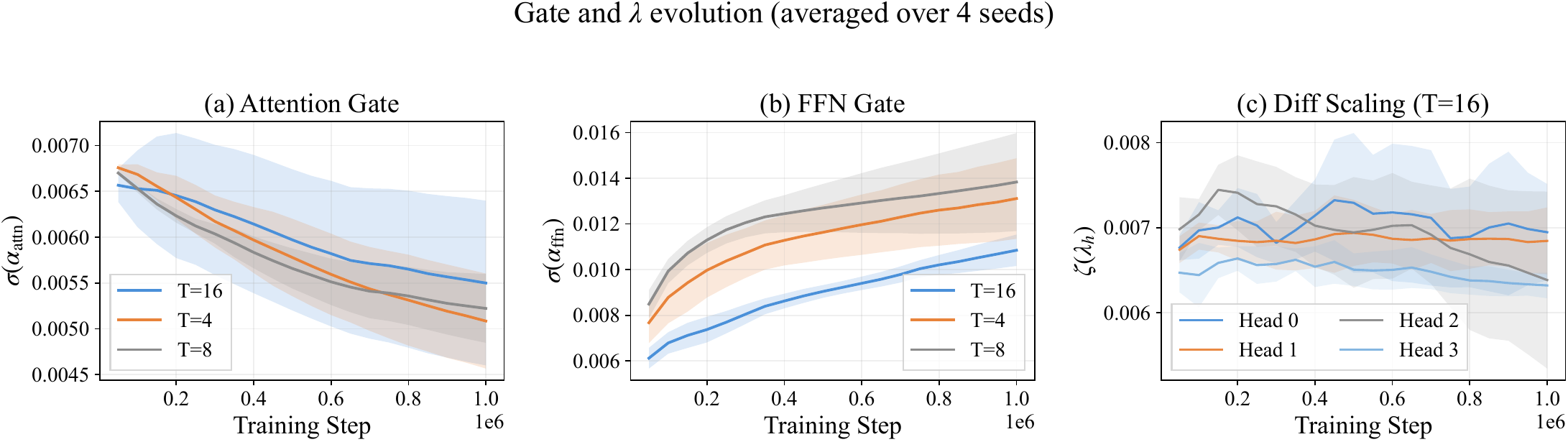}
\caption{\textbf{Gate and difference-scaling evolution on AntMaze-Large.} Both $\sigma(\alpha_\mathrm{attn})$ and $\sigma(\alpha_\mathrm{ffn})$ stay below a small threshold throughout the entire training run, and per-head $\zeta(\lambda_h)$ values remain within a narrow band. Small gates do not make the refinement a small perturbation of $\phi(g)$, since the post-norm block rescales $\phi(g)$ even at closed gates (\cref{prop:placement,tab:geometry}).}
\label{fig:gates}
\vspace{-0.5cm}
\end{figure}
\cref{fig:attention} compares the attention distributions of standard cross-attention and \DGCA{} on four Cube-Double samples against the corresponding difference maps. Both are nearly one-hot. \DGCA{} attends to a different token than cross-attention on every sample, so the bias does shift the attended position, yet neither method selects the argmax of $\Delta$. On this task the difference map therefore acts as a soft regularizer that displaces the head from the similarity-maximizing token, rather than as a literal pointer to the misaligned region.

\cref{fig:gates} plots the gate and $\zeta(\lambda)$ trajectories that \cref{sec:attribution} discusses. Both stay near their initial values. \cref{tab:geometry} asks what that implies for the two quantities the theory assumes. We load frozen Cube-Double checkpoints and run forward, JVP and VJP passes only, with no training. The per-sample local spectral norm of $\mathrm{d}\delta_\ell / \mathrm{d}\phi$ comes from power iteration at real $(s, g)$ pairs drawn from the offline value-goal distribution.

\begin{table}[h]
\centering
\caption{\textbf{Block geometry on frozen Cube-Double checkpoints.} One seed, $2048$ $(s, g)$ pairs and $50$ power iterations per entry. Here $c_\ell$ is the \emph{local} spectral norm of $\mathrm{d}\delta_\ell / \mathrm{d}\phi$ at the data points, hence a lower bound on the constant assumed by \cref{thm:sufficiency}. The shut-gate columns set every $\alpha$ to $-30$, for which $\sigma(\alpha) \approx 10^{-13}$. All six gates satisfy $\sigma(\alpha) < 0.013$ in both checkpoints, so every converged entry is produced by a nearly closed gate.}
\label{tab:geometry}
\small
\begin{tabular}{lcccc}
\toprule
& \multicolumn{2}{c}{\textbf{post-norm}} & \multicolumn{2}{c}{\textbf{pre-norm}} \\
\cmidrule(lr){2-3}\cmidrule(lr){4-5}
\textbf{Quantity} & converged & gates shut & converged & gates shut \\
\midrule
$c_\ell$ at $T = 16$ & 1.014 & 1.011 & 6.507 & 0.000 \\
$c_\ell$ at $T = 8$  & 1.008 & 1.046 & 3.560 & 0.000 \\
$c_\ell$ at $T = 4$  & 1.083 & 1.051 & 3.662 & 0.000 \\
worst local $c_\ell$ & 8.278 & 1.061 & 160.6 & 0.000 \\
\midrule
$B / \|\phi(g)\|_2$ & 0.991 & 0.971 & 0.931 & 0.000 \\
$\|\phi_{\textup{\DAGR{}}}(g \mid s)\|_2 / \|\phi(g)\|_2$ & 0.018 & 0.046 & 0.709 & 1.000 \\
$\|\phi(g)\|_2$ & \multicolumn{2}{c}{84.27} & \multicolumn{2}{c}{4.06} \\
\bottomrule
\end{tabular}
\end{table}

Three readings follow. First, \cref{prop:placement} holds numerically. At shut gates the pre-norm block returns its input to seven digits, so $c_\ell$, $B$ and the norm ratio read $0$, $0$ and $1$ exactly. The post-norm block instead returns a vector whose norm is $4.6\%$ of the base, which leaves $B$ at $97\%$ of $\|\phi(g)\|_2$. Second, the contraction hypothesis of \cref{thm:sufficiency} fails at the converged weights under both placements. A local estimate can falsify and cannot confirm, and here it falsifies. A nearly closed gate does not imply a small residual, because the branch that the gate multiplies carries a large Lipschitz constant of its own. Third, the norm ratio separates the two placements in the same direction as success does, at $0.018$ against $0.709$. The Dual distance $\psi(s)^\top \phi(g)$ is homogeneous of degree one in $\phi(g)$, so that ratio is the scale the value head reads, and \cref{rem:opt_gap} reports its behavioural counterpart.

The base norms differ by a factor of twenty across the two checkpoints, because the placement feeds back into how the base encoder trains. Only the two ratio rows are therefore comparable across columns. These numbers come from a single seed on one manipulation task, so we read them as a check on \cref{prop:placement} and not as a measurement of the navigation regime.

\subsection{Why Visual-Puzzle Remains at Zero}
\label{sec:puzzle_analysis}

\begin{table}[t]
\centering
\caption{\textbf{Visual-Puzzle: representation-level methods are uniformly zero}, whereas early fusion alone (\citet{park2026dual}, Table~3) already achieves substantial success. Hierarchical planning (HIQL) is also far above zero.}
\label{tab:puzzle_comparison}
\begin{tabular}{lcc}
\toprule
\textbf{Method} & \textbf{Puzzle-3$\times$3} & \textbf{Puzzle-4$\times$4} \\
\midrule
Late Fusion (Dual) & 0 & 0 \\
Dual + CA & 0 & 0 \\
Dual + \DAGR{} & 0 & 0 \\
Early Fusion (Orig) & \underline{22} & \textcolor{orange}{\textbf{65}} \\
HIQL~\citep{park2023hiql} & \textcolor{orange}{\textbf{73}} & \underline{60} \\
\bottomrule
\end{tabular}
\vspace{-0.5cm}
\end{table}
\cref{tab:puzzle_comparison} supports the encoder-level account in \cref{sec:attribution}. Two remedies would apply. Cross-attention inside the CNN~\citep{huang2025early} preserves the correspondence that pooling destroys, and hierarchical planning sidesteps the requirement for it. Both are complementary to \DAGR{}, since they act upstream of $\phi$.

\subsection{Per-Task Learned $\zeta(\lambda)$ and Fusion Weights}

\begin{table}[h]
\centering
\caption{Per-task average $\zeta(\lambda)$ (mean across $H = 4$ heads and $L = 3$ scales at the end of training, one seed per task). Initial value $\zeta(\lambda_0) = \zeta(-5) \approx 0.0067$.}
\label{tab:lambda_analysis}
\begin{tabular}{lc}
\toprule
\textbf{Task} & \textbf{Average $\zeta(\lambda)$} \\
\midrule
Navigation (AntMaze-Large) & 0.0066 \\
Single-Object Manipulation (Cube-Single) & 0.0227 \\
Multi-Object Manipulation (Cube-Double) & 0.0066 \\
Scene Arrangement & 0.0592 \\
Visual Navigation (Visual-AntMaze-Large) & 0.0074 \\
Visual Manipulation (Visual-Cube-Double) & 0.0069 \\
\bottomrule
\end{tabular}
\end{table}

The learned $\zeta(\lambda)$ values stay at or barely above initialization on four of six tasks, with the largest deviation on Scene and a smaller one on Cube-Single. The naive reading that manipulation needs more difference bias than navigation is not supported: AntMaze-Large and Cube-Double share essentially the same value. This is consistent with the component ablation result that removing the difference bias does not eliminate the navigation gains. Practically, the role of $\Delta$ in MS-\DGCA{} is best read as a structured inductive bias on the cross-attention scoring at initialization rather than as a heavily optimized contribution at convergence.

\begin{table}[h]
\centering
\caption{Learned fusion weights $\mathrm{softmax}(w)$ at the end of training. Fine $= T_1 = 16$, Medium $= T_2 = 8$, Coarse $= T_3 = 4$. Initial weights are uniform at $1/3 \approx 0.333$.}
\label{tab:fusion_weights}
\begin{tabular}{lccc}
\toprule
\textbf{Task} & \textbf{Fine ($T_1$)} & \textbf{Medium ($T_2$)} & \textbf{Coarse ($T_3$)} \\
\midrule
Navigation (AntMaze-Large) & 0.360 & 0.257 & 0.383 \\
Single-Object Manipulation (Cube-Single) & 0.423 & 0.216 & 0.361 \\
Multi-Object Manipulation (Cube-Double) & 0.317 & 0.312 & 0.371 \\
Scene Arrangement & 0.224 & 0.144 & 0.632 \\
Visual Navigation (Visual-AntMaze-Large) & 0.683 & 0.137 & 0.180 \\
Visual Manipulation (Visual-Cube-Double) & 0.374 & 0.328 & 0.299 \\
\bottomrule
\end{tabular}
\end{table}

Fusion weights tell a less clean story than the design intuition in \cref{sec:multiscale} alone would suggest. Scene drifts strongly toward the coarse scale, Visual Navigation toward the fine scale, and the remaining tasks stay close to uniform. We take this as evidence that the optimal analysis granularity is genuinely data-driven and that the role of the multi-scale architecture is to let the model select rather than to enforce a particular bias.

\subsection{Gate and $\lambda$ Evolution on Additional Tasks}

\begin{figure}[h]
\centering
\includegraphics[width=0.98\textwidth]{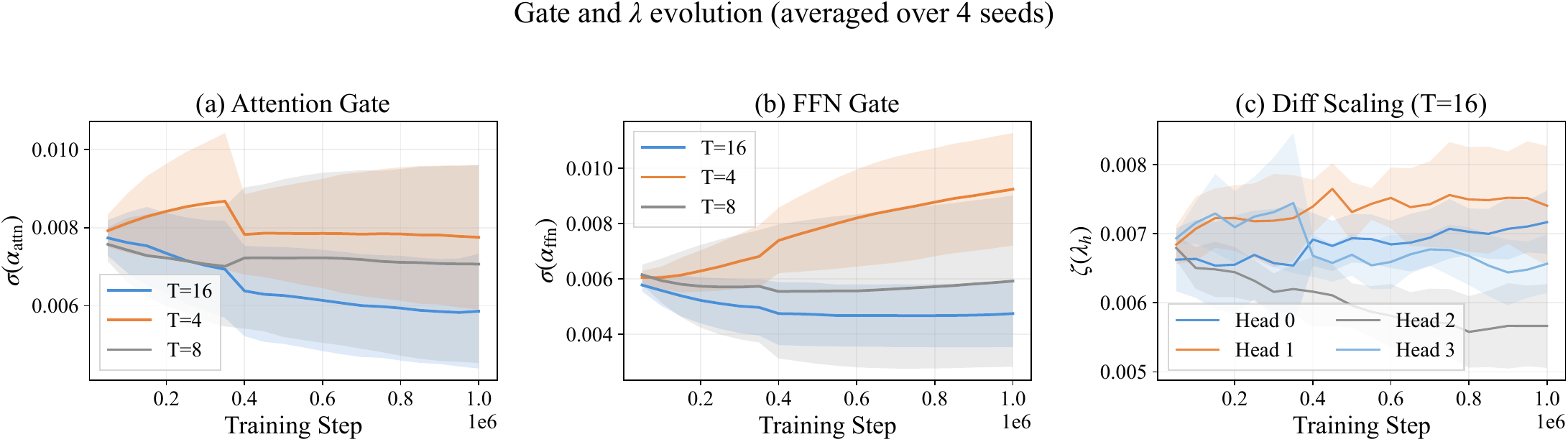}
\caption{Gate and difference-scaling evolution on Cube-Double. The same pattern as on AntMaze-Large (\cref{fig:gates}, main text) holds: all gates and per-head $\zeta(\lambda_h)$ values move only marginally from their initial values throughout training.}
\label{fig:gates_cube}
\end{figure}

\begin{figure}[h]
\centering
\includegraphics[width=0.98\textwidth]{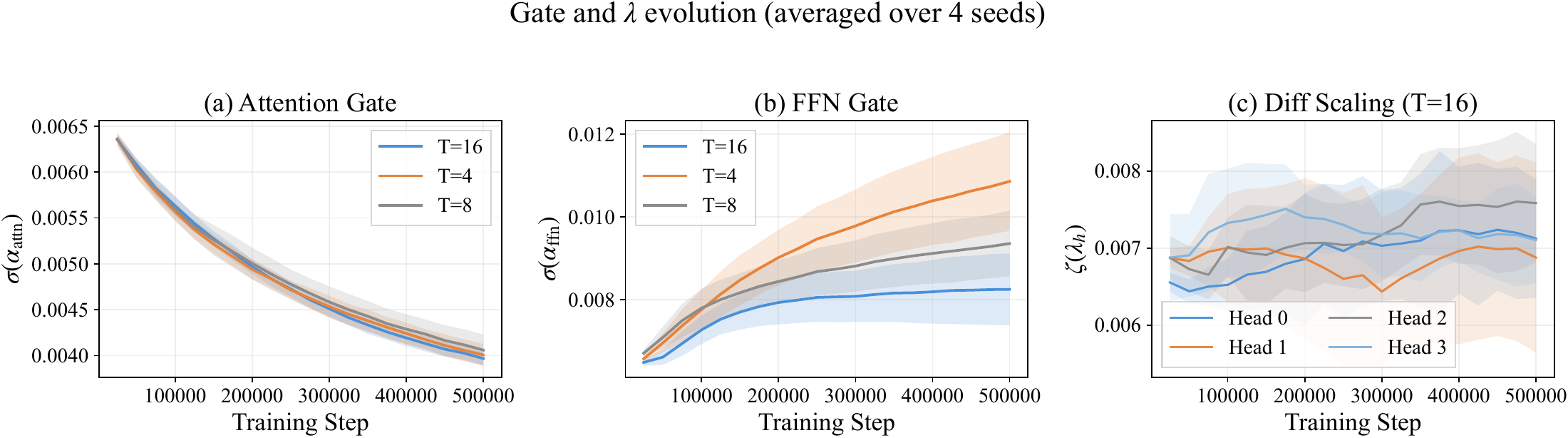}
\caption{Gate and difference-scaling evolution on Visual-Cube-Double. The pattern is qualitatively identical to the state-based case.}
\label{fig:gates_visualcube}
\end{figure}

\subsection{Learning Curves}
\label{sec:curves}

\begin{figure}[t]
\centering
\includegraphics[width=0.98\textwidth]{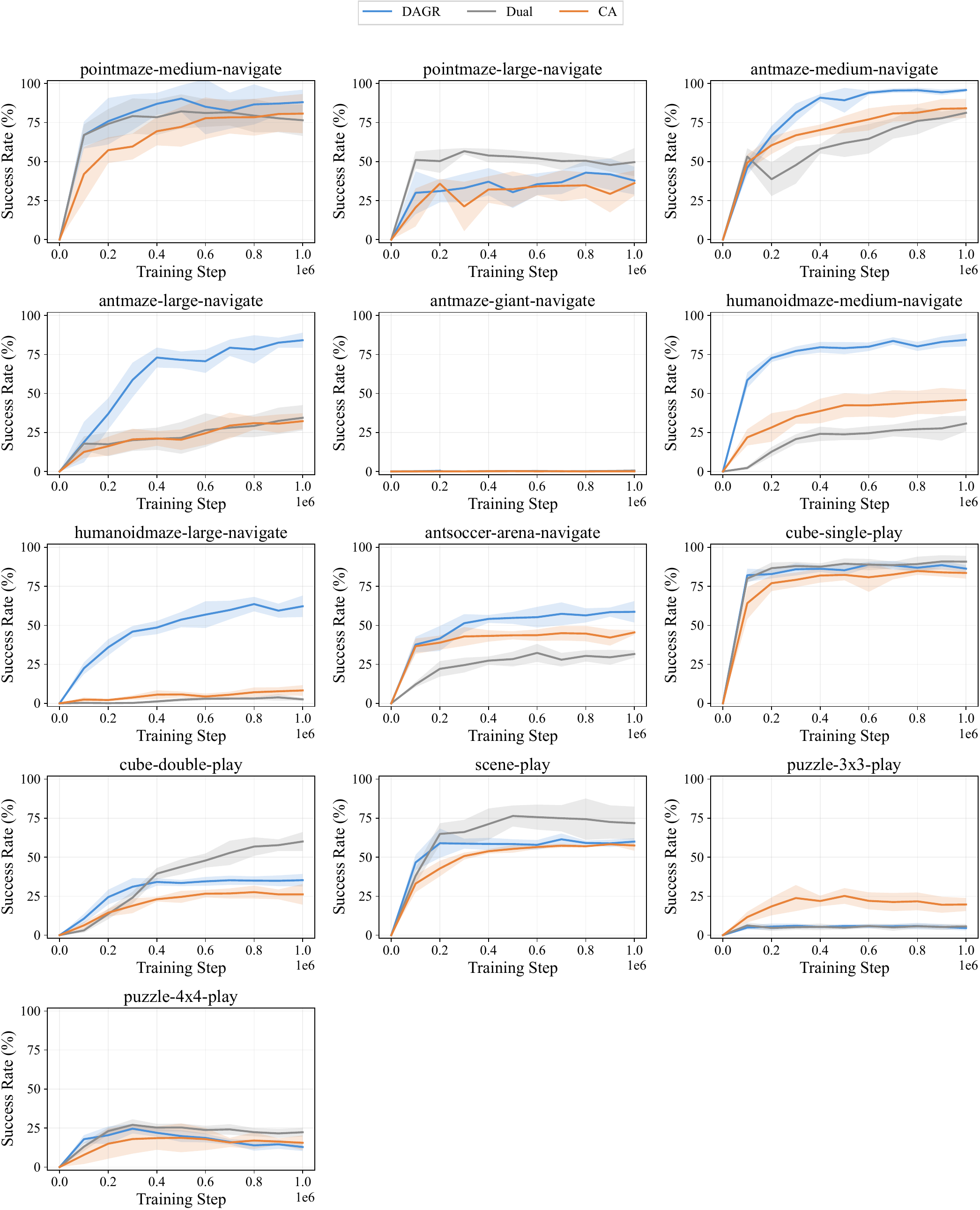}
\caption{\textbf{Learning curves on the thirteen state-based OGBench tasks.} \DAGR{}, Dual and standard CA, identified in the legend, trained for $10^6$ steps with eight seeds. The shaded band is the $95\%$ confidence interval. On navigation (first eight panels) \DAGR{} separates early from both baselines and stays ahead, except on PointMaze and AntMaze-Giant. On manipulation (next three panels) the curves overlap on Cube-Single and Dual leads on Cube-Double and Scene, and on the puzzles (last two panels) success stays low for all three methods.}
\label{fig:curves_state}
\end{figure}

\begin{figure}[t]
\centering
\includegraphics[width=0.98\textwidth]{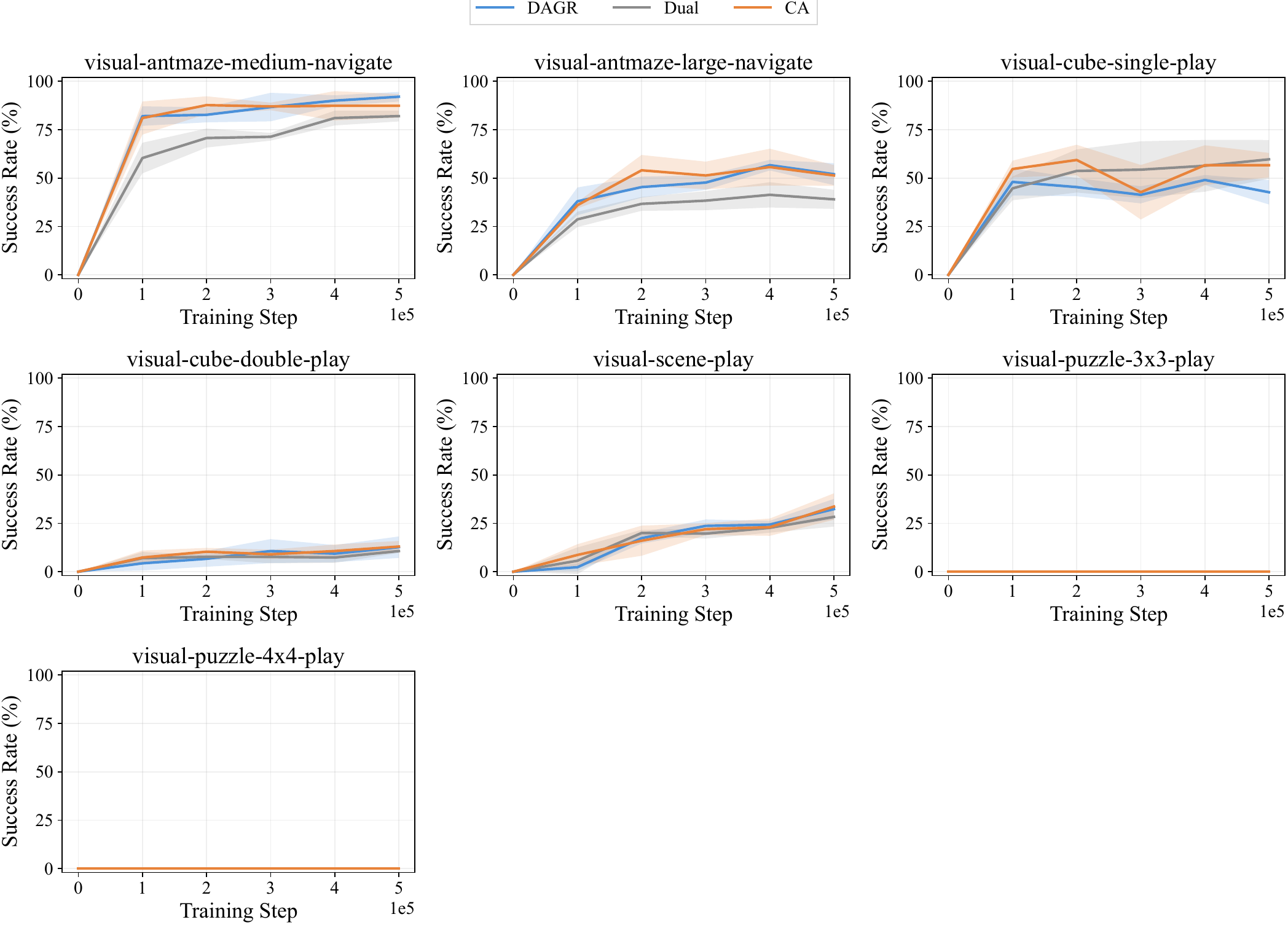}
\caption{\textbf{Learning curves on the seven visual OGBench tasks.} Same color scheme as \cref{fig:curves_state}, trained for $5 \times 10^5$ steps with four seeds. Visual-AntMaze (first two panels) reproduces the navigation pattern observed in the state-based setting, with \DAGR{} separating from Dual within the first quarter of training. Visual manipulation panels overlap except on Cube-Single, where \DAGR{} ends lowest. Visual-Puzzle tasks (last two panels) remain at zero for every method.}
\label{fig:curves_visual}
\end{figure}

\cref{fig:curves_state,fig:curves_visual} show the full training trajectories on the state-based and visual OGBench suites. The two views tell consistent stories, and together they sharpen the picture from the aggregate tables.

On the state-based suite (\cref{fig:curves_state}), the navigation panels share a clear shape. \DAGR{} departs from the two baselines early, typically within the first fifth of training, and the gap widens or stabilizes rather than closing. The most pronounced separation occurs on HumanoidMaze-large, AntMaze-large, and HumanoidMaze-medium, where the Dual baseline plateaus low and \DAGR{} continues to climb. Standard CA lands between Dual and \DAGR{} on most navigation tasks, five to thirteen points above Dual in \cref{tab:ablation_ca}. Introducing a state-goal interaction is therefore not sufficient on its own. Of the components separating \DAGR{} from CA, the gate matters most, as \cref{tab:ablation_components,tab:sensitivities} show, and the curves confirm that the separation opens within the first fifth of training rather than late. On PointMaze-medium and PointMaze-large, whose state space is too small to compress, the ordering is mixed, with \DAGR{} ahead on the medium maze and Dual ahead on the large one. The manipulation panels (Cube-Single, Cube-Double, Scene) overlap on Cube-Single, whereas Dual stays clearly above \DAGR{} and CA on Cube-Double and Scene. This matches the aggregate table and is now shown to be stable across training rather than an artifact of the final checkpoint. The puzzle panels (Puzzle-3x3, Puzzle-4x4) stay low for all three methods throughout training, with no sign of a late breakthrough. No pixel encoder is involved here, so unlike \cref{sec:puzzle_analysis} the ceiling is not encoder-level; every method in \cref{tab:state_results}, the raw-goal Orig included, stays at or below $23$.

On the visual suite (\cref{fig:curves_visual}), the navigation behavior of \cref{fig:curves_state} reproduces: on both Visual-AntMaze variants, \DAGR{} pulls away from Dual within the first quarter of training and the separation persists to the end. Visual manipulation curves cluster tightly on Cube-Double and Scene, whereas on Cube-Single \DAGR{} ends below both baselines, again consistent with the aggregate result. Visual-Puzzle remains at zero throughout training for every method, which is the strongest evidence we have that the bottleneck on these tasks is not something a goal-representation refinement, however structured, can fix.

Two cross-cutting observations emerge from comparing the two figures. First, the gains from \DAGR{} appear early and grow with training rather than emerging only at convergence. They do not wait for the gates, which barely move (\cref{fig:gates}). Under post-norm, $B$ in \cref{thm:approximation_bound} is of order $\|\phi(g)\|_2$ even at closed gates (\cref{prop:placement}), so the block departs from $\phi(g)$ from the first step. Second, where \DAGR{} does not improve over Dual the behavior splits into two regimes. On PointMaze-large, whose state space is small, \DAGR{} trails Dual within noise. On Cube-Double and Scene \DAGR{} settles below Dual by mid-training and stays there, which rules out a late-stage optimization artifact and points at the architectural cause identified in \cref{sec:conclusion}.

\end{document}